\documentclass[10pt]{article} 
\usepackage[preprint]{tmlr}

\usepackage{amsmath,amsfonts,bm}

\def\eqref#1{equation~\ref{#1}}

\def\1{\bm{1}}

\DeclareMathAlphabet{\mathsfit}{\encodingdefault}{\sfdefault}{m}{sl}
\SetMathAlphabet{\mathsfit}{bold}{\encodingdefault}{\sfdefault}{bx}{n}

\usepackage{graphicx}
\usepackage{array}
\usepackage{booktabs} 
\usepackage{soul}
\usepackage{bbm}
\usepackage{xcolor}
\usepackage{subcaption}
\usepackage{url}            
\usepackage{amsfonts}       
\usepackage{nicefrac}       
\usepackage{microtype}      
\usepackage{tikz}
\usepackage[dvipsnames]{xcolor}
\usepackage{amsmath}
\usepackage{amssymb}
\usepackage{amsthm}
\usepackage{multirow}
\usepackage{multicol}
\usepackage{hyperref}
\hypersetup{
    hidelinks, 
}
\usepackage{cleveref}
\usepackage{url}

\newtheorem{theorem}{Theorem}
\newtheorem{proposition}[theorem]{Proposition}

\setstcolor{red}

\title{On the Role of the Projector in {Contrastive} Self-Supervised Learning: Last-Layer Rank Dynamics Drive Representation Quality \thanks{Under review at Transactions on Machine Learning Research (TMLR)"}}

\author{\name Siladittya Manna \email smanna@hkbu.edu.hk, siladittyam@iisc.ac.in \\
      \addr Hong Kong Baptist University, Hong Kong, China \thanks{Majority of the work for this manuscript was done while Dr. Siladittya Manna was working as a Senior Research Assistant at Hong Kong Baptist University.}\\
      \addr Indian Institute of Science, Bengaluru, India
      \AND
      \name Priyangshu Mandal \email priyangshu.mandal@kgpian.iitkgp.ac.in \\
      \addr Indian Institute of Technology Kharagpur, India
      \AND
      \name Umapada Pal \email umapada@isical.ac.in\\
      \addr Indian Statistical Institute, Kolkata, India
      \AND
      \name Saumik Bhattacharya \email saumik@ece.iitkgp.ac.in\\
      \addr Indian Institute of Technology Kharagpur, India
      }

\def\month{05}  
\def\year{2026} 
\def\openreview{\url{https://openreview.net/forum?id=XXXX}} 

\begin{document}

\maketitle

\begin{abstract}
The dimensional collapse of representations in self-supervised {contrastive} learning is an ever-present issue. One notable technique to prevent such a collapse of representations is using a multi-layered perceptron network called Projector. In several works, the projector has been found to heavily influence the quality of representations learned in a self-supervised {contrastive} pre-training task. However, the question still lingers. \textit{What role does the projector play?} Assuming the projector mitigates dimensional collapse, what prevents the terminal layer of the base encoder from functioning as the projector in the absence of an explicit multi-layer perceptron (MLP) head? In this work, we intend to study what happens inside the projector by examining the rank dynamics of the same and the encoder through empirical study and analysis. Through mathematical analysis, we observe that the effect of rank reduction predominantly occurs in the last layer. Motivated by this insight, we propose a weight regularization strategy applied specifically to the last layer. We demonstrate that this targeted approach yields better performance than applying orthogonal weight regularization across the entire network (WeRank), both with and without a projector. Our method improves Top-1 accuracy by more than 1\% on SimCLR on the ImageNet100 dataset and consistently outperforms baseline SimCLR variants on CIFAR datasets, supporting our interpretation of the projector’s role.
\end{abstract}

\section{Introduction}
\label{sec:intro}

Self-supervised learning aims to learn representations without any human annotations. Recent works like SimCLR \citep{chen2020simclr}, MoCov2 \citep{chen2020mocov2}, DCL \citep{yeh2021dcl}, BYOL \citep{grilljb2020byol}, Barlow Twins \citep{ZbontarJMLD21barlowtwins}, etc. present frameworks which allow learning of representations which are similar for semantically similar samples. However, this objective may lead to a complete collapse of representations when the representations of all samples get mapped trivially to a single point in the representation space.

Various techniques, such as using a large batch size \citep{chen2020simclr}, momentum encoder \citep{chen2020mocov2, grilljb2020byol}, stop gradient \citep{chen2020simsiam}, feature whitening \citep{bardes2022vicreg,ZbontarJMLD21barlowtwins} and clustering \citep{caron2020swav}, have been used to prevent the complete collapse of representations. However, contrastive self-supervised learning still suffers from dimensional collapse, where the embedding vectors only span a lower-dimensional subspace. Dimensional collapse
occurs when the variance of information along some dimensions becomes insignificant. We avoid saying that variance will be zero because information content along any dimension can never be entirely zero in practical terms. 

In \citet{Hua2021OnFD}, the author discusses that dimensional collapse is mainly related to a strong correlation between information flowing through different dimensions. This challenging issue of dimensional collapse has also been addressed in works like \citet{balestriero2022connoncon}, RankMe \citep{garrido2023rankme}, DirectCLR \citep{jing2022directclr} and WeRank \citep{pasand2024werank}. These works also stress the importance of full-rank representations for better performance on downstream tasks. However, WeRank does not provide any mathematical insight into the dimensional collapse of representation. In DirectCLR, the attempt at investigating the causes of the dimensional collapse is limited to toy examples, and only uses a truncated vector for training, leaving the last few dimensions non-trainable. This, however, is not fully capable of preventing dimensional collapse. We instead use the full output vector for both training and evaluation as well. \textcolor{black}{Additionally, WeRank restricts the set of possible learnable functions to a subset satisfying the orthogonality condition of the weight matrices, which causes an over-regularization effect and prevents learning of expressive features. We find several pieces of evidence in the literature, not limited to self-supervised learning, which state that anisotropy is essential for capturing non-uniform data structures, with neural rendering models using it to resolve directional ambiguity and disambiguate geometry from complex appearances for high-fidelity 3D reconstructions \citep{Wang_2025_nerf, GaoC0Y25anisdf}. In parallel, research on Large Language Models demonstrates that decreasing isotropy via the I-STAR regularizer improves semantic performance by facilitating data clustering and reducing the intrinsic dimensionality of representations \citep{RudmanE24istar}. {Recently, LeJePA \citep{randall2025lejepa} also stresses that embeddings should be isotropic and that anisotropy amplifies both bias and variance.} We show that not enforcing the feature maps to be isotropic via orthogonal weight regularisation, while only regularising the last layer, improves the weight spectrum and reduces the dimensional collapse effect in self-supervised contrastive learning when trained without a projector.} Furthermore, we show that, unlike WeRank \citep{pasand2024werank}, it is not necessary to apply the weight regularisation on the whole network, thereby reducing the computation overhead from $\mathcal{O}(L\cdot n^3)$ to $\mathcal{O}(n^3)$, where $L$ is the scalar factor that comes naturally, as shown in the later section Sec. \ref{subsec:roleproj}.

{Even though the embedding covariance matrix remains ill-conditioned when a projector is used, we understand that the projector plays an important role in reducing dimensional collapse.} The role of the projector has been studied previously in works like \cite{gupta2022understanding, song2023sparseness, xue2024investigating}. However, none of the above works explores \textit{what} the fundamental causes are behind the effect of dimensional collapse. 

In this work, we first empirically verify the decorrelating effect of InfoNCE loss. Next, we try to determine what happens inside the projector and its role in self-supervised contrastive learning. We further investigate the phenomenon occurring in the encoder layer that causes degradation in performance in the case of dimensional collapse resulting from negligible eigenvalues of the embedding covariance matrix, both with or without a projector. \ul{Finally, we employ a simple strategy for self-supervised learning, both with and without using a projector, which verifies our mathematical conclusion about the role of weight norm.} We summarize our contributions as follows:

\begin{itemize}
\setlength{\itemsep}{-0.3em}
    \item We investigate the role of the projector in self-supervised contrastive learning in the light of dimensional collapse. To our knowledge, this is \textcolor{black}{one of the }first works to do so. 
    \item We further investigate the phenomenon in the encoder when not using a projector for contrastive self-supervised pre-training, giving us more insight into the phenomenon of dimensional collapse.
    \item Based on our findings, we propose a simple strategy to improve performance in contrastive self-supervised pre-training by applying weight regularization only on the last layer.
    \item The proposed regularization strategy outperforms the contemporary regularization strategy WeRank on benchmark datasets CIFAR10, CIFAR100 and ImageNet100.
\end{itemize}

\section{Related Works}
\label{sec:relworks}

SSL methods take different approaches to prevent a complete collapse of representations. Instance discrimination methods like SimCLR \citep{chen2020simclr}, MoCov2 \citep{chen2020mocov2}, and DCL \citep{yeh2021dcl} use repulsion between negative samples to prevent complete collapse. However, in addition to the negative repulsion in the InfoNCE loss, they also use a projector which projects the encoder output representations into a lower-dimensional space before computing the InfoNCE loss. Methods like DeepCluster \citep{caron2018deepc} and SwAV \citep{caron2020swav} use a clustering-based instance-group discrimination approach. However, dimensional collapse persists according to \citet{garrido2023rankme} and \citet{jing2022directclr}.

Architectures similar to the above are also seen in dimension contrastive methods like BYOL \citep{grilljb2020byol}, where the extra predictor for predicting the output of the projector from the momentum updated target encoder and $l2$-normalization prevents complete collapse. SimSiam \citep{chen2020simsiam}, on the other hand, uses a stop-gradient method to prevent the same. WMSE \citep{ermolov2021wmse}, ZeroCL \citep{zhang2022zerocl} uses feature whitening to prevent collapse. 

Non-contrastive methods like Barlow Twins \citep{ZbontarJMLD21barlowtwins} aim to decorrelate the feature dimensions to reduce redundancy in the output embeddings, thereby preventing dimensional collapse. However, Barlow Twins fails to perform without the projector, as we will see in the later subsections. VICReg \citep{bardes2022vicreg} uses a covariance term in the loss to do feature decorrelation like Barlow Twins. However, according to \citet{garrido2023rankme}, even these methods are not free from dimensional collapse.


\citet{Hua2021OnFD} discusses that the strong correlation between dimensions of the representation vector is the primary cause of dimensional collapse, and uses feature decorrelation to prevent it and improve performance. \textcolor{black}{\citet{balestriero2022connoncon} also uses a decorrelation loss like VICReg as a method to prevent dimensional collapse and learn optimal representations}. \citet{gupta2022understanding} shows that a projector prevents low-rank backbone features, thereby preventing dimensional collapse. However, it does not explore the reason behind it. This work primarily discusses that a learnable projection head is a way of mitigating the shortcomings of contrastive loss and helps in learning generalizable representations. A detailed discussion of the relationship between downstream performance and embedding rank is also presented in \citet{garrido2023rankme}. WeRank \citep{pasand2024werank} uses the same feature decorrelation strategy to deduce that the weight norm of each layer should be as close to the identity matrix as possible to prevent dimensional collapse.

DirectCLR \citep{jing2022directclr} achieved considerable success in preventing collapse. This work mainly proposed two findings as the possible causes of dimensional collapse: (1) implicit regularization due to over-parametrization of networks, and (2) strong augmentations. However, in terms of performance (linear evaluation accuracy), it falls short of SimCLR with a non-linear projector. 

{\color{black}
DINO \citep{caron2021dino} introduced self-distillation without labels, where a student network learns from a momentum-updated teacher using normalized feature matching, enabling Vision Transformers to learn semantic features without supervision. iBOT \citep{zhou2021ibot} extended DINO by combining self-distillation with masked image modelling, allowing simultaneous global and local representation learning through patch-level prediction. DINOv2 \citep{maxime2024dinov2} further refined the framework with large-scale curated data, improved regularization, and stronger transferability, yielding universal visual features competitive with supervised models. I-JEPA \citep{assran2023ijepa} departed from contrastive and pixel-level objectives by predicting high-level latent representations of masked regions, emphasizing abstraction and context understanding. Together, these methods progressively evolve self-supervised learning from instance discrimination toward semantically rich, transferable, and predictive representations.

Other related works, such as VCReg \citep{zhu2023vcreg}, extend the idea of weight and feature regularization by encouraging high-variance and low-covariance representations to enhance feature diversity and prevent neural collapse. This approach aligns conceptually with our method and WeRank \citep{pasand2024werank}, as both aim to improve representation quality and transferability through more balanced and decorrelated feature learning.
}

\section{Methodology}
\label{sec:methodology}

\subsection{Preliminaries}
\label{subsec:prelims}

In this work, we consider SSL pre-training with SimCLR as the baseline. Let us denote $f$ and $g$ as the encoder and the projector, respectively. The encoder output and the projector output embeddings are denoted by $h = f(x)$ and $z = g(f(x))$, respectively, where $x$ denotes the input sample. The total number of dimensions in the encoder output embedding is given by $D$. $d_0$ and $d_r$ denote the number of dimensions of the encoder output embedding, which are trainable and non-trainable or fixed to a constant value. To learn the representations, the InfoNCE loss is given by, 

\begin{minipage}{\linewidth}
\begin{equation}
    \label{eqn:infonce}
    \mathcal{L}_{infonce} = - \mathop{\mathbb{E}}_{i} \left[ \frac{exp(s_{ii+})}{exp(s_{ii+}) + \sum_{\substack{j=1\\j\neq i}}^{B} exp(s_{ij})}\right]
\end{equation}
\end{minipage}

\noindent where $s_{ii+}$ and $s_{ij}$ are the cosine similarity between the projector output embeddings of the samples of positive pair $(x_i,x_{i+})$ obtained by augmentations applied on the sample $x_i$, and the samples of negative pair $(x_i, x_j)$, respectively, and $B$ denotes the batch size.

In the later subsections, we divide the output embeddings into 2 parts, which we refer to as trainable and non-trainable dimensions. We define the trainable part of an embedding to consist of those dimensions through which the gradient propagation is allowed to flow. At the same time, the non-trainable part of the embedding means the opposite. 

{\color{black}

\subsection{Definitions}
\label{subsec:defs}

{
\paragraph{(D1) Dimensional Collapse:} Dimensional collapse occurs when representation vectors $h \in \mathbb{R}^D$ are effectively constrained to an $r$-dimensional linear subspace $\mathbb{R}^r \subset \mathbb{R}^D$ with $r < \min(N, D)$ ($N$ denotes the sample count), characterized by eigenvalues of the representation covariance matrix $\Sigma_h = \mathrm{Cov}(h)$ decaying to near-zero magnitudes ($\lambda_i(\Sigma_h) \le \epsilon$ for $i \in \{r+1, \dots, D\}$). Unlike complete collapse (where all embeddings map to a single point), dimensional collapse selectively suppresses variance along specific orthogonal directions, restricting representations to a low-rank subspace.
}

\paragraph{(D2) Information Bottleneck:} The information bottleneck (IB) is a principle for representation learning that aims to extract a compressed representation of a variable that is as predictive as possible of a target variable. The IB framework is based on finding a compressed representation, $T$, of an input variable, $X$, that preserves the maximum relevant information about a target variable, $Y$. This is expressed through a constrained optimization problem. The goal is to find the representation $T$ that maximizes the relevance $\mathcal{I}(T; Y)$ for a given compression level $\mathcal{I}(T; X)$. This can be expressed as a Lagrangian optimization problem: $\min _{p(t|x)}\mathcal{I}(T;X)-\beta \mathcal{I}(T;Y)$, where $\beta$ is the Lagrangian multiplier.

{
\paragraph{(D3) High-level representations:} High-level representations refer to feature spaces that encode abstract, class-discriminative, and view-invariant semantic attributes of the input data while discarding task-irrelevant, low-level spatial variations.
In deep architectures, feature extraction is hierarchical; earlier layers capture local spatial primitives, whereas the terminal layers (such as layer $L$) map these primitives into high-level semantic representations. Consequently, weight degradation or rank collapse at layer $L$ directly impairs the network's capacity to output well-conditioned high-level features suitable for downstream tasks. 
}



\textbf{Relevance to Downstream Tasks}: The principle of information bottleneck is utilised in the downstream task to discard irrelevant information while retaining useful information related to the downstream task. The high-level representations, that is, the task-specific representations in the deeper layers, are also relevant to the downstream task. Finally, the dimensional collapse, which can occur for both the self-supervised pre-training and supervised training stages, causes the learning to occur in a lower-dimensional subspace rather than in the high-dimensional embedding space. A larger utilisation of the learning subspace results in better performance in the downstream task.

}

\subsection{Theoretical Setup and Scope}
\label{sec:theoretical_setup}

In this subsection, we formalize the modelling assumptions under which the theoretical analysis in the subsequent sections is carried out. This setup is not a repetition of the preliminaries in Sec. \ref{subsec:prelims}, but rather a restriction of scope that specifies which components of the network and training objective are explicitly modelled and which are abstracted away.

\paragraph{Architectural assumptions.}
We consider an InfoNCE-based contrastive SSL framework with a convolutional neural network (CNN) encoder \( f \), optionally followed by a projection head \( g \), as introduced in Sec. \ref{subsec:prelims}. While a projector is not strictly required for all self-supervised learning paradigms (e.g., DirectCLR \citep{jing2022directclr}), it is a standard architectural component in InfoNCE-based contrastive frameworks such as SimCLR and has been shown to play an important role in stabilising training and mitigating representation collapse. Accordingly, the presence of a projector is assumed in the general setup, while its absence is treated as a special case that is analyzed explicitly in later sections.

\paragraph{Layer-wise scope of analysis.}
Although the encoder \( f \) consists of multiple convolutional layers, our analysis is intentionally restricted to the \emph{final two convolutional layers of the encoder} and the layers of the projector. All earlier layers are treated as a black box that produces intermediate feature representations with well-defined second-order statistics (i.e., finite covariance).

This restriction is motivated by empirical observations showing that dimensional collapse and rank degradation emerge most prominently at the layer whose output is directly optimized by the contrastive loss (shown in later sections). Consequently, all propositions in this paper concern the behaviour of:
(i) the final encoder layer when no projector is used, or
(ii) the projector layers when a projection head is present.

Throughout the paper, we denote by layer \( L \) the layer whose output embedding is used to compute the InfoNCE loss. In this work, layer \(l-1\) is termed a shallower layer with respect to \(l\), while the layer \(l\) is deeper with respect to the layer \(l-1\). While there is no specific threshold in the literature to specify which layers are deep or shallow, by the term ``deep'' we will consider layers which are close to the final layer $L$, while ``shallow'' layers indicate those closer to the input layer.

\paragraph{Linearization and analytical simplifications.}
For analytical tractability, the layers under consideration (i.e., the last two encoder layers and the projector layers) are modelled as linear transformations parameterized by weight matrices. Non-linear activations, batch normalization, and skip connections are omitted from the theoretical analysis. This simplification allows us to isolate the structural relationship between embedding covariance, weight norms, and dimensional collapse.

A convolutional layer without non-linear activation defines a linear operator. Let $X \in \mathbb{R}^{C_{\mathrm{in}} \times H \times W}$ and
$Y \in \mathbb{R}^{C_{\mathrm{out}} \times H' \times W'}$ be the input and output feature maps, respectively. After vectorization, there exists a structured sparse matrix $W_{\mathrm{conv}} \in \mathbb{R}^{(C_{\mathrm{out}}H'W') \times (C_{\mathrm{in}}HW)}$ such that $\mathrm{vec}(Y) = W_{\mathrm{conv}} \mathrm{vec}(X)$. The sparsity and block Toeplitz structure of $W_{\mathrm{conv}}$ (Appendix Sec. \ref{app:conv_fc}) arise from local connectivity and weight sharing, but do not affect the linear-algebraic arguments used in this work.

Importantly, this linearization is used only for theoretical reasoning; all empirical results in this paper are obtained using full CNN architectures with non-linearities and normalization layers intact.


Under this setup, the theoretical results in the following Sections should be interpreted as local, layer-wise characterizations of dimensional collapse in InfoNCE-based contrastive learning, conditioned on the architectural and modelling assumptions stated above. 

\subsection{Motivation}
\label{subsec:motivation}

In this work, the main motivation is to study the phenomenon occurring inside the projector in the self-supervised contrastive learning scenario and what happens in the absence of it. In DirectCLR \citep{jing2022directclr}, it is stated that in instance discrimination-based contrastive learning, even though the presence of positive and negative samples should prevent the dimensional collapse of representations intuitively, it still occurs.

\begin{figure}[!ht] 
     \begin{center}
    \begin{subfigure}[b]{0.45\linewidth}
        \centering
     \includegraphics[width = \linewidth]{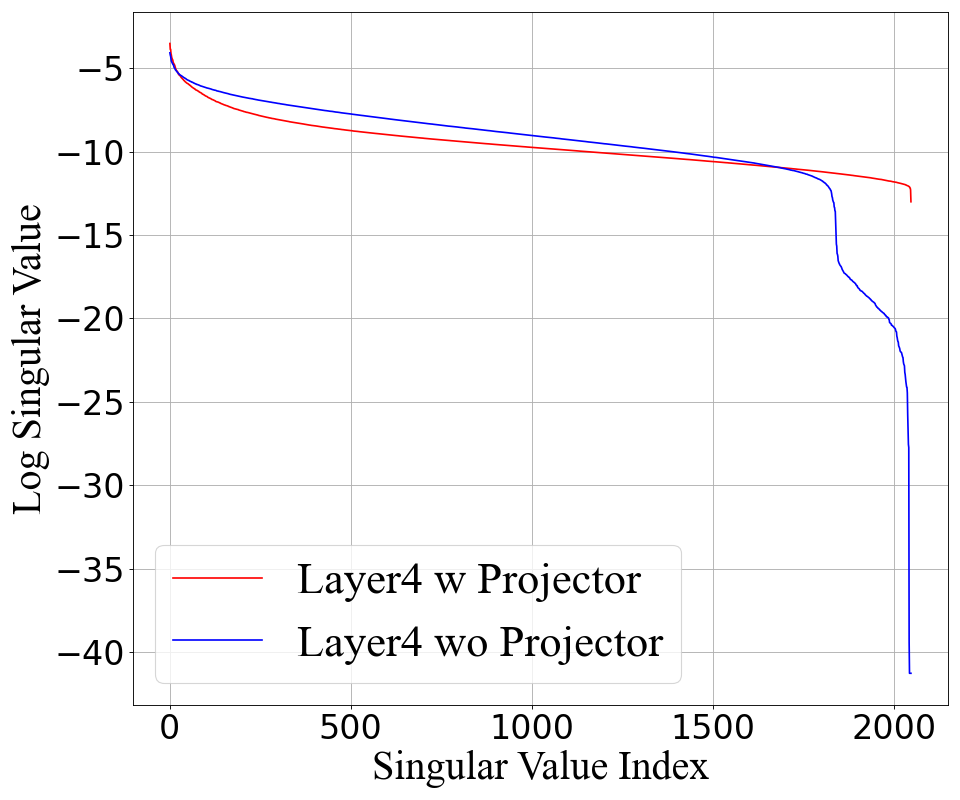}
     \caption{}
     \label{fig:infoncedecorr}
     \end{subfigure}
     \hfill
     \begin{subfigure}[b]{0.45\linewidth}
     \centering
         \includegraphics[width = \linewidth]{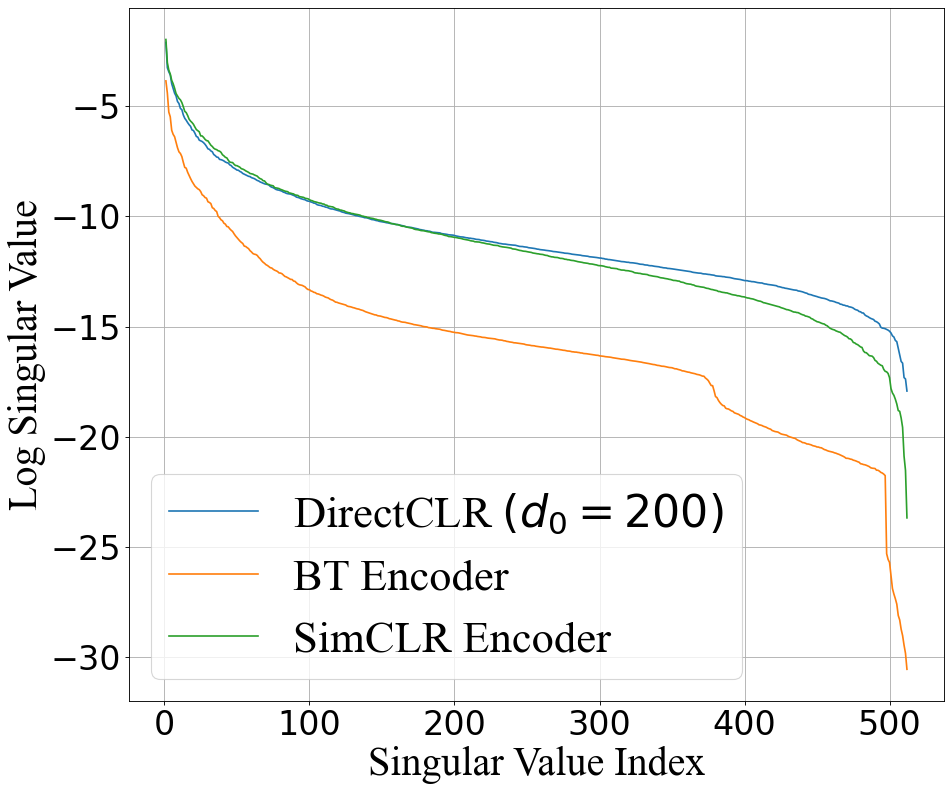}
     \vspace{-1em}
     \caption{}
     \label{fig:encoders_BT_direct_simclr}
     \end{subfigure}
     \end{center}
     \vspace{-1em}
     \caption{(a) Singular value plots of the covariance matrix of ResNet50 encoder output embeddings pre-trained on ImageNet100 using SimCLR with and without a projector. `Layer4' indicates the last layer in the ResNet50 encoder. `\textcolor{black}{blue}': without (wo) projector, `\textcolor{black}{red}':with projector. (b) Singular value plots of Barlow Twins and SimCLR encoders compared with DirectCLR. \textcolor{black}{The plot (a) exhibits that without the projector, the singular values of the covariance matrix drop sharply. A similar observation is also found in Barlow Twins (b), while the vanilla SimCLR and DirectCLR method prevents the sharp decline.}}
 \end{figure}

We find this to be true empirically as shown in Fig. \ref{fig:infoncedecorr}, where we observe that the magnitudes of the sorted eigenvalue spectrum dip considerably when the encoder is trained without a non-linear projector than when trained with one. Similar findings are also reported in \citet{gupta2022understanding}. Furthermore, methods using feature decorrelation to prevent dimensional collapse, like \citet{balestriero2022connoncon} or \citet{Hua2021OnFD}, still suffer from dimensional collapse. This is primarily due to the low-rank embeddings of \textcolor{black}{shallower} layers, that is, from the encoder \citep{pasand2024werank}.
To determine the role of the projector, we empirically study whether the InfoNCE loss has a decorrelating effect. Then we try to analyze the dynamics of the projector through rank decomposition of the covariance matrix and how it prevents dimensional collapse.

\begin{figure*}
    \centering
    \begin{subfigure}[b]{0.245\linewidth}
        \centering
        \includegraphics[width =\linewidth]{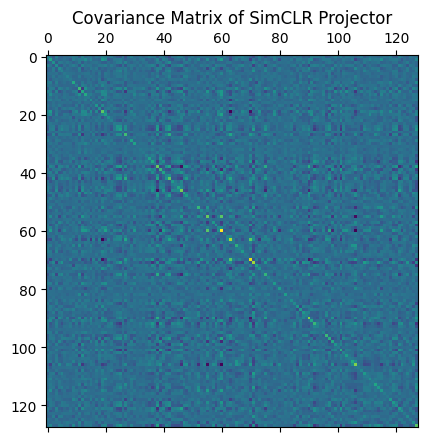}
        \caption{CIFAR100}
        \label{fig:c100_cov_mat_proj}
    \end{subfigure}%
    \hfill
    \begin{subfigure}[b]{0.245\linewidth}
        \centering
        \includegraphics[width =\linewidth]{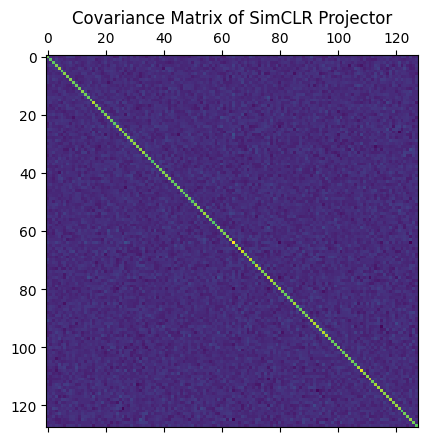}
        \caption{ImageNet100}
        \label{fig:in100_cov_mat_proj}
    \end{subfigure}%
    \hfill
    \begin{subfigure}[b]{0.245\linewidth}
        \centering
        \includegraphics[width = \linewidth]{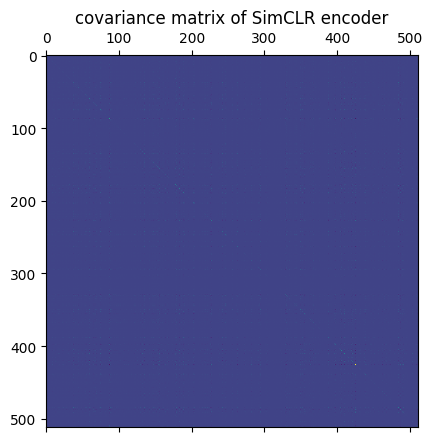}
        \caption{CIFAR100}
        \label{fig:c100_cov_mat_enc}
    \end{subfigure}%
    \hfill
    \begin{subfigure}[b]{0.245\linewidth}
        \centering
        \includegraphics[width = \linewidth]{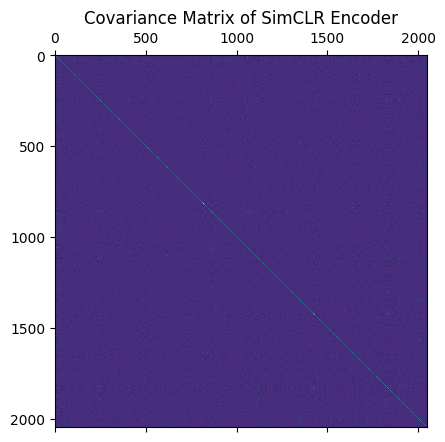}
        \caption{ImageNet100}
        \label{fig:in100_cov_mat_enc}
    \end{subfigure}
    \caption{Covariance matrices of output embedding of the projector for SimCLR trained on (a) CIFAR100 and (b) ImageNet100. Covariance matrices of embeddings from the SimCLR encoder trained on (c) CIFAR100 and (d) ImageNet100. \textcolor{black}{On both CIFAR100 and ImageNet100, we can observe the projector output embedding exhibiting low covariance in the off-diagonal components, which points towards the decorrelation effects of the InfoNCE loss (a and b). Furthermore, the decorrelation effect is propagated partially to the encoder output embeddings, as evident from the uniform nature of the covariance matrix values (c and d).} Best viewed at 300\%.}
    \label{fig:cov_mat_proj_enc}
\end{figure*}%

\subsection{Does InfoNCE have a decorrelating effect?}
\label{subsec:infoncedecorr}

According to \citet{ZhangZZPYK22}, InfoNCE also acts as a decorrelating loss, similar to Barlow Twins \citep{ZbontarJMLD21barlowtwins} or \citet{balestriero2022connoncon}. 
In Fig. \ref{fig:c100_cov_mat_proj} and \ref{fig:in100_cov_mat_proj}, we show the covariance matrix of the output feature dimensions. From the covariance matrix of the embeddings of the CIFAR100 and ImageNet100 datasets, we can see that the magnitudes of the diagonal elements of the covariance matrix are much higher than the non-diagonal ones. This shows that the InfoNCE loss has a decorrelating effect, as shown in \citet{ZhangZZPYK22}. However, from Fig. \ref{fig:c100_cov_mat_enc} and \ref{fig:in100_cov_mat_enc}, we see that the diagonal nature of the covariance matrix of the encoder output embeddings is not present. This proves that even if the loss enforces feature decorrelation on the projector output embeddings, it is possible to obtain low-rank output embeddings from the encoder.

\subsection{Understanding the events in Projector in case of Dimensional Collapse}
\label{subsec:roleproj}

 It is empirically observed in DirectCLR \citep{jing2022directclr} that InfoNCE loss fails to properly optimize the parameters of a network without a projector and results in dimensional collapse. Barlow Twins (BT) \citep{ZbontarJMLD21barlowtwins} performs better than most contrastive learning frameworks on benchmark datasets, but not when implemented without a projector, even though a decorrelation loss is applied.
 The singular value spectrum in Fig. \ref{fig:encoders_BT_direct_simclr} plots the singular values of BT (w/o projector), SimCLR (w/o projector), and DirectCLR. We can see that the dimensional collapse effect in BT is greater than in SimCLR w/o projector, even though it is trained directly using a decorrelation-based loss. According to \citet{ZhangZZPYK22}, the projector is essential for a decorrelation-based framework too, even though both InfoNCE and the loss used in \citet{ZbontarJMLD21barlowtwins} have a decorrelating component. So, the question arises, \textbf{\textit{what exactly happens after the addition of a non-linear projector towards the prevention of dimensional collapse?}}. 

 
  \paragraph{\textbf{A Linear Algebraic perspective:}} In RankMe \citep{garrido2023rankme} and DirectCLR \citep{jing2022directclr}, the authors have shown that without a projector, the embeddings from the pre-trained encoder have a low rank. Does a low-rank embedding indicate that the useful information can be approximated using fewer dimensions? But then, \textbf{\textit{why does it lead to worse performance, if that is the case?}}. How is it different from the information bottleneck theory of the projector? \citep{wang2025projection}.
  



\subsubsection{\textcolor{black}{Weight Norm Analysis}}
\label{subsubsec:wtnormanal}
 
It is important to note that in DirectCLR \citep{jing2022directclr}, a part of the output vector $z$ is left unchanged; that is, the kernels leading to the unchanged part of $z$ still have the randomly initialized weights at the end of pre-training. Now, these randomly initialized weights have non-zero variance. However, when the rank of the encoder output embeddings is reduced, it practically means that the variance of the information along those dimensions is very low. A very low variance means there is almost no useful information available in that dimension. Thus, when not using a projector, the reduction in rank in the last layer embedding covariance matrix indicates that there is little variance of information along some of the embedding dimensions. Consequently, assuming that the input to the last layer is full rank and well-conditioned, it indicates that the norm of the weights in the last layer $\lVert W^i_l \rVert^2 < \epsilon$, where $\epsilon$ is very small, and $\lambda_i(Cov(h)) \rightarrow 0$, where $W^i_l$ is the layer weights corresponding to the $i$-th output embedding dimension of layer $l$. 

\noindent

{
\begin{proposition}
[Collapsed dimensions imply vanishing weight norms]
Let $x \in \mathbb{R}^{D_i}$ be a random vector with covariance
$\Sigma_x := \mathrm{Cov}(x) \succ 0$.
Let $W \in \mathbb{R}^{D_o \times D_i}$ be a linear map and define
$h = Wx$, with $h_i = W_i x$, where $W_i$ denotes the $i$-th row of $W$. Furthermore, let $\operatorname{Cov}(h) = \Sigma_h = P \Lambda P^T$ be the spectral decomposition of the representation covariance matrix, where $P=\left[p_1,\hdots,p_{D_o}\right] \in \mathbb{R}^{D_o \times D_o}$ is an orthogonal matrix of eigenvectors $(P^TP = PP^T = I_{D_o})$ and $\Lambda = \operatorname{diag}(\lambda_1, \hdots, \lambda_{D_o})$ is the diagonal matrix of eigenvalues. Then, an eigenvalue $\lambda_i(\Sigma_h) = 0$ if and only if $W^T p_i = 0$, where $p_i$ is the $i$-th eigenvector (column) of $P$. Consequently, dimensional collapse $(\lambda_i(\Sigma_h)=0)$ cannot be induced by $X$ and is driven solely by $W$.


\end{proposition}

\noindent
\textbf{\textit{Proof}}: We can prove the above proposition by a simple deduction. Let $x \in \mathbb{R}^{D_i}$ be the embeddings with dimensions $N \times D_i$, and $W \in \mathbb{R}^{D_o \times D_i}$ be the weight matrix with dimensions $D_o \times D_i$. To consider only a single dimension $i$, we take the $i$-th row of the weight matrix as $W^i \in \mathbb{R}^{1 \times D_i}$. Let the covariance of $x$ be defined as 
\begin{equation}
    \Sigma_x = Cov(x) = \mathop{\mathop{\mathbb{E}}}_{x}\left[\left(x-\mathop{\mathbb{E}}_x[x])\right)\left(x-\mathop{\mathbb{E}}_x[x])\right)^T\right]
\end{equation}
It is to be noted that, since we are dealing with the last layer only, we assume that the input to the last layer does not have any collapsed dimension, and $\Sigma_x$ is positive definite with the smallest eigenvalue $\lambda_{min} > 0$ and is well conditioned, that is $\lambda_{min} \approx \lambda_{max} > 0$.
\begin{equation}
    h_i = W^i_l x_l = W^i_l.\left( W_{l-1} x_{l-1} \right) = \left(W^i_l W_{l-1} \right) x_{l-1}
\end{equation}

Because $\Sigma_x = \operatorname{Cov}(X) \in \mathbb{R}^{D_i \times D_i}$ is real, symmetric, and Positive Definite ($\Sigma_x \succ 0$), its eigenvalues are strictly positive:\begin{equation}\lambda_{\max}(\Sigma_x) \ge \dots \ge \lambda_i(\Sigma_x) \ge \dots \ge \lambda_{\min}(\Sigma_x) > 0\end{equation}

The covariance of the representation $h = WX$ is given by:
\begin{equation}
\operatorname{Cov}(h) = \operatorname{Cov}(Wx) = W \operatorname{Cov}(x) W^\top = W \Sigma_x W^\top
\end{equation}

Expressing $\operatorname{Cov}(h)$ via its spectral decomposition $\Sigma_h = \operatorname{Cov}(h) = P \Lambda P^\top$, we establish the matrix identity:
\begin{equation}
P \Lambda P^\top = W \Sigma_x W^\top
\end{equation}

Premultiplying both sides by $P^\top$ and postmultiplying by $P$, and using the orthogonality property $P^\top P = P P^\top = I_d$, we isolate the diagonal eigenvalue matrix $\Lambda$:
\begin{equation}
\Lambda = P^\top W \Sigma_x W^\top P
\end{equation}

Since $\Lambda = P^\top \operatorname{Cov}(h) P$ is diagonal with entries $\Lambda_{ii} = \lambda_i(\Sigma_h)$, and the $i$-th column of $P$ is the eigenvector $\mathbf{p}_i$, taking the $(i,i)$-th entry of both sides gives:
\begin{equation}
    \lambda_i(\Sigma_h) = \mathbf{p}_i^\top W \Sigma_x W^\top \mathbf{p}_i = (W^\top \mathbf{p}_i)^\top \Sigma_x (W^\top \mathbf{p}_i)
    \label{eqn:lambda_sigh}
\end{equation}

Define the projection vector $\mathbf{q}_i \in \mathbb{R}^n$ as:
\begin{equation}
\mathbf{q}_i = W^\top \mathbf{p}_i 
\end{equation}

Substituting $\mathbf{q}_i$ back into Equation \ref{eqn:lambda_sigh} gives the quadratic form:
\begin{equation}
\lambda_i(\Sigma_h) = \mathbf{q}_i^\top \Sigma_x \mathbf{q}_i
\label{eqn:lambda_sigh2}
\end{equation}

By the Rayleigh-Ritz Theorem \citep{Trefethen2022NLA}, for any vector $\mathbf{q}_i \in \mathbb{R}^n$ and positive-definite matrix $\Sigma_x \succ 0$:
\begin{equation}
\lambda_{\min}(\Sigma_x) \|\mathbf{q}_i\|_2^2 \le \mathbf{q}_i^\top \Sigma_x \mathbf{q}i \le \lambda_{\max}(\Sigma_x) \|\mathbf{q}_i\|_2^2
\label{eqn:uplowbounds}
\end{equation}

Combining Equations \ref{eqn:lambda_sigh2} and \ref{eqn:uplowbounds} yields the sandwich inequality:
\begin{equation}
\lambda_{\min}(\Sigma_x) \|\mathbf{q}_i\|_2^2 \le \lambda_i(\Sigma_h) \le \lambda_{\max}(\Sigma_x) \|\mathbf{q}_i\|_2^2
\end{equation}

We evaluate both directions of the logical equivalence $\lambda_i(\Sigma_h) = 0 \iff \mathbf{q}_i = \mathbf{0}$:

$(\Leftarrow)$  Sufficiency: If $\mathbf{q}_i = \mathbf{0}$, then $\lambda_i(\Sigma_h) = \mathbf{0}^\top \Sigma_x \mathbf{0} = 0$.

$(\Rightarrow)$ Necessity: Suppose $\lambda_i(\Sigma_h) = 0$. From the lower bound in equation (12):

\begin{equation}
\lambda_{\min}(\Sigma_x) \|\mathbf{q}_i\|_2^2 \le 0
\end{equation}

Because $\Sigma_x \succ 0$, we have $\lambda_{\min}(\Sigma_x) > 0$. Since $\Vert{}\mathbf{q}_i\Vert{}_2^2 \ge 0$, inequality (13) forces:
\begin{equation}
\|\mathbf{q}_i\|_2^2 = 0 \implies \mathbf{q}_i = \mathbf{0}
\end{equation}

Therefore,
\begin{equation}
\lambda_i(\Sigma_h) = 0 \quad \Longleftrightarrow \quad \mathbf{q}_i = \mathbf{0} \quad \Longleftrightarrow \quad W^\top \mathbf{p}_i = \mathbf{0}
\end{equation}

Since $P$ is an orthogonal matrix, its $i$-th column $\mathbf{p}_i$ is an orthonormal eigenvector satisfying $\Vert{}\mathbf{p}_i\Vert{}_2 = 1 \neq 0$. The condition $W^\top \mathbf{p}_i = \mathbf{0}$ requires the non-zero eigenvector $\mathbf{p}_i \in \mathbb{R}^{D_o}$ to lie in the nullspace of $W^\top$:
\begin{equation}
\mathbf{p}_i \in \operatorname{null}(W^\top)
\end{equation}

By the Rank-Nullity Theorem, the dimension of this nullspace is:
\begin{equation}
\dim(\operatorname{null}(W^\top)) = d - \operatorname{rank}(W)
\end{equation}

We analyze two distinct cases for $W$: 

\begin{itemize}
    \item Full Row Rank ($\operatorname{rank}(W) = d$):

    The nullspace $\operatorname{null}(W^\top) = \{\mathbf{0}\}$. Since $\Vert{}\mathbf{p}_i\Vert{}_2 = 1$, no eigenvector can belong to $\operatorname{null}(W^\top)$. Thus, $\mathbf{q}_i = W^\top \mathbf{p}_i \neq \mathbf{0}$ for all $i$, which guarantees that $\lambda_i(A) > 0$ for all $i = 1, \dots, d$.

    \item Rank-Deficient ($\operatorname{rank}(W) = r < d$):

The nullspace dimension is $\dim(\operatorname{null}(W^\top)) = d - r > 0$. Consequently, exactly $d - r$ orthogonal eigenvectors $\mathbf{p}_i$ span this nullspace and satisfy $W^\top \mathbf{p}_i = \mathbf{0}$, forcing exactly $d - r$ eigenvalues to collapse ($\lambda_i(\Sigma_h) = 0$).
\end{itemize}

Because $\lambda_{\min}(\Sigma_x) > 0$ strictly prevents the positive-definite input distribution $X$ from collapsing $\mathbf{q}_i^\top \Sigma_x \mathbf{q}_i = 0$ for any non-zero $\mathbf{q}_i$, output eigenvalue collapse ($\lambda_i(\Sigma_h) = 0$) occurs if and only if $W^\top \mathbf{p}_i = \mathbf{0}$. Hence, dimensional collapse is solely driven by $W$.

While the above proposition establishes that exact collapse ($\lambda_i(\Sigma_h) = 0$) on well-conditioned data is strictly weight-induced ($\operatorname{rank}(W) < D_o$), real-world feature distributions frequently lie on low-dimensional submanifolds, yielding ill-conditioned input covariance matrices where $\lambda_{\min}(\Sigma_x) = \epsilon \approx 0$.

In this regime, the lower bound in Equation \ref{eqn:uplowbounds} collapses ($\lambda_{\min}(\Sigma_x) \to 0$), removing the guarantee that full-rank weights ($\operatorname{rank}(W) = D_o$) protect against representation collapse. To investigate how input ill-conditioning interacts with $W$ to induce practical representation collapse ($\lambda_i(\Sigma_h) \le \epsilon$), we relax the well-conditioned assumption on $\Sigma_x$ and analyze the joint spectral alignment between $W$ and the eigenspaces of $\Sigma_x$.
}

\paragraph{Key Takeaway 1: Weights with near-zero norm prevent learning of high-level representations, provided input is well-conditioned} Thus, if the norm and the variance of the weights \textcolor{black}{corresponding to the collapsed dimensions} become close to zero, it prevents the kernels \textcolor{black}{in the deeper layers} in the backbone encoder from learning class-specific high-level representations, consequently hampering the downstream performance, as the presence of lower values in the weight variances reduces their representation learning capacity.

\begin{figure}[!ht]
  \centering
  \begin{subfigure}[b]{0.24\linewidth}
      \centering
    \includegraphics[width = 0.95\linewidth]{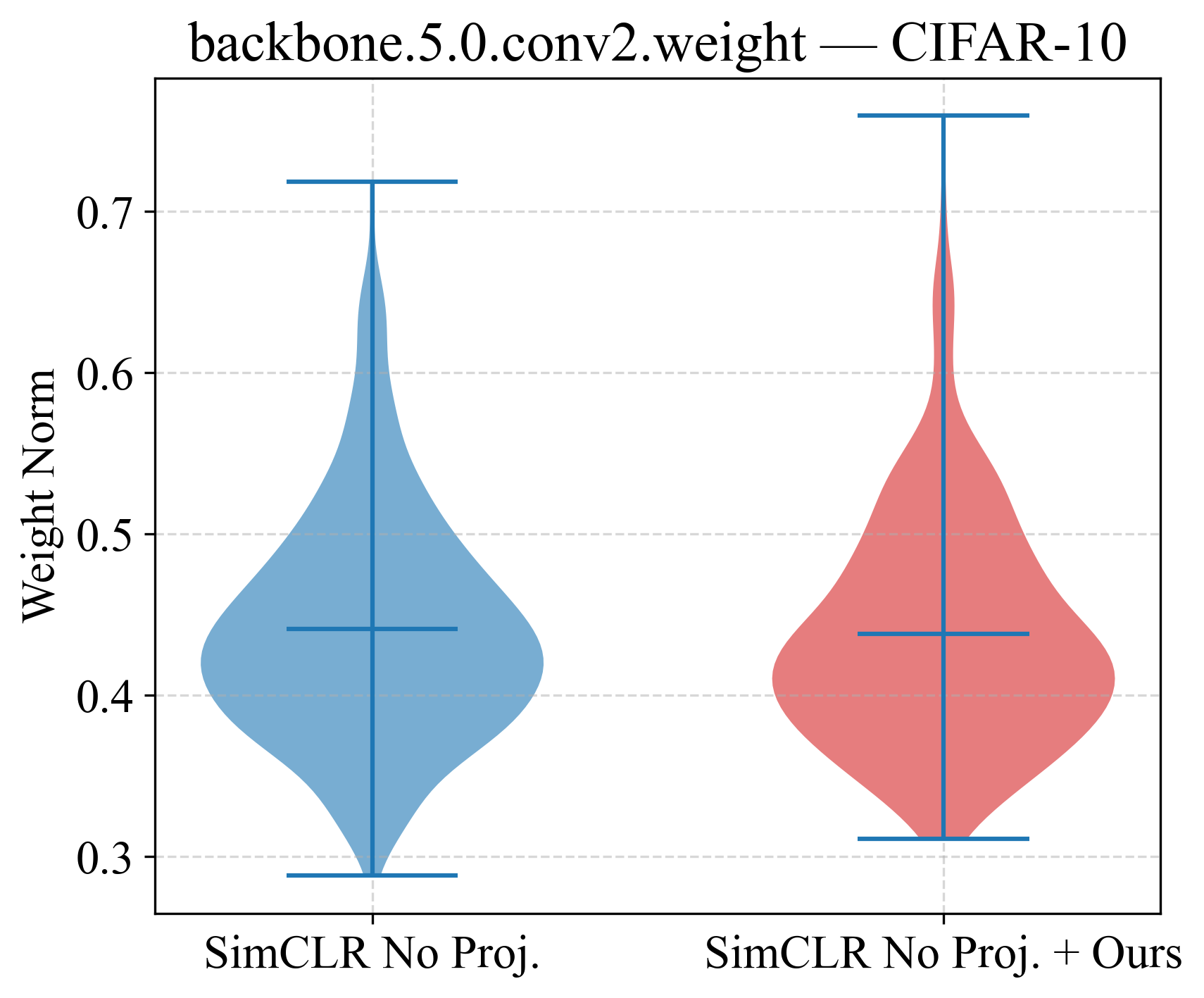}
      \caption{CIFAR10 - Last Layer - 1st Block - 2nd Conv}
      \label{fig:vio_proj_c10_502}
  \end{subfigure}
  \hfill
  \begin{subfigure}[b]{0.24\linewidth}
      \centering
    \includegraphics[width = 0.95\linewidth]{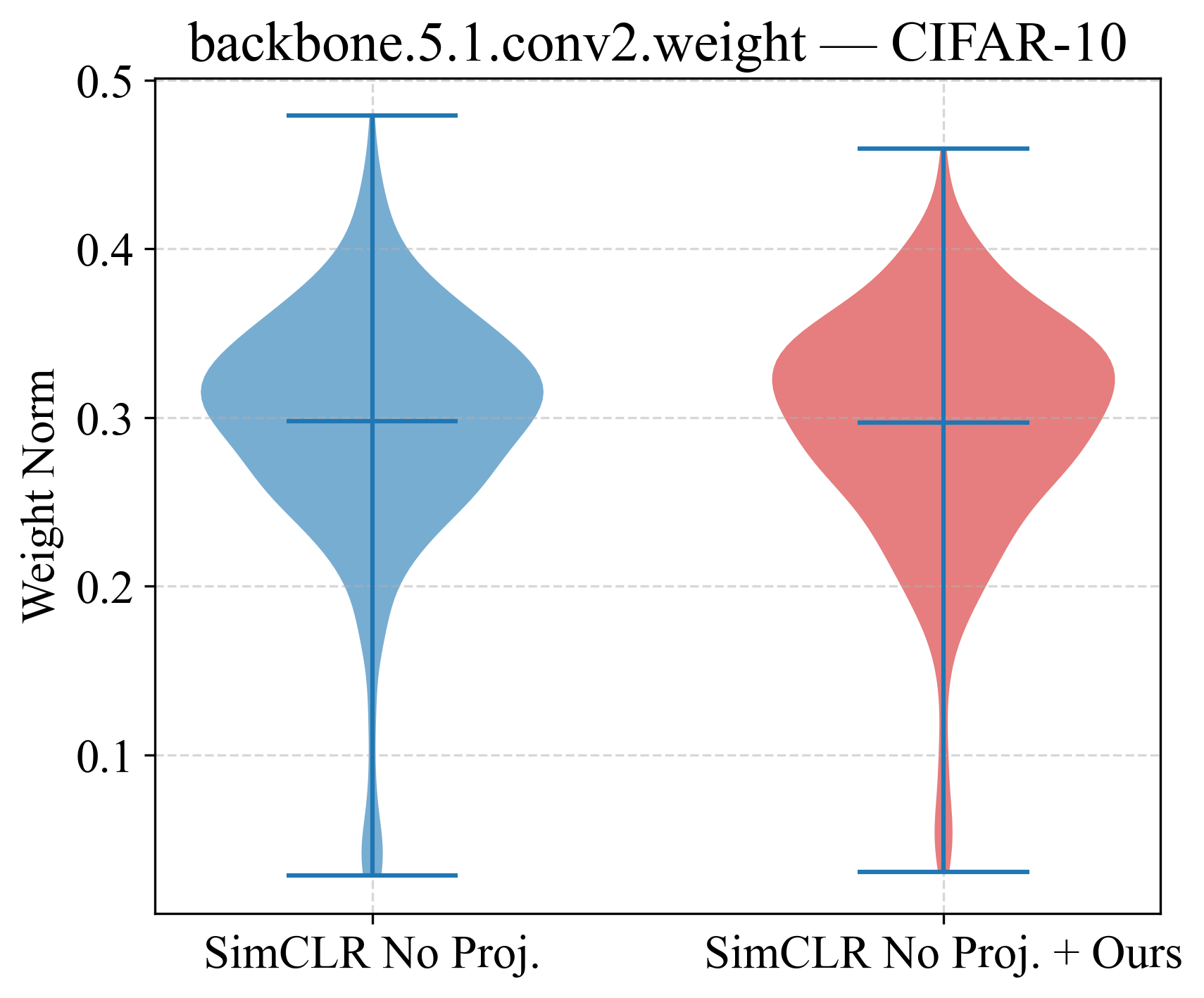}
      \caption{CIFAR10 - Last Layer - 2nd Block - 2nd Conv}
      \label{fig:vio_proj_c10_512}
  \end{subfigure}%
  \hfill
  \begin{subfigure}[b]{0.24\linewidth}
      \centering
    \includegraphics[width = 0.95\linewidth]{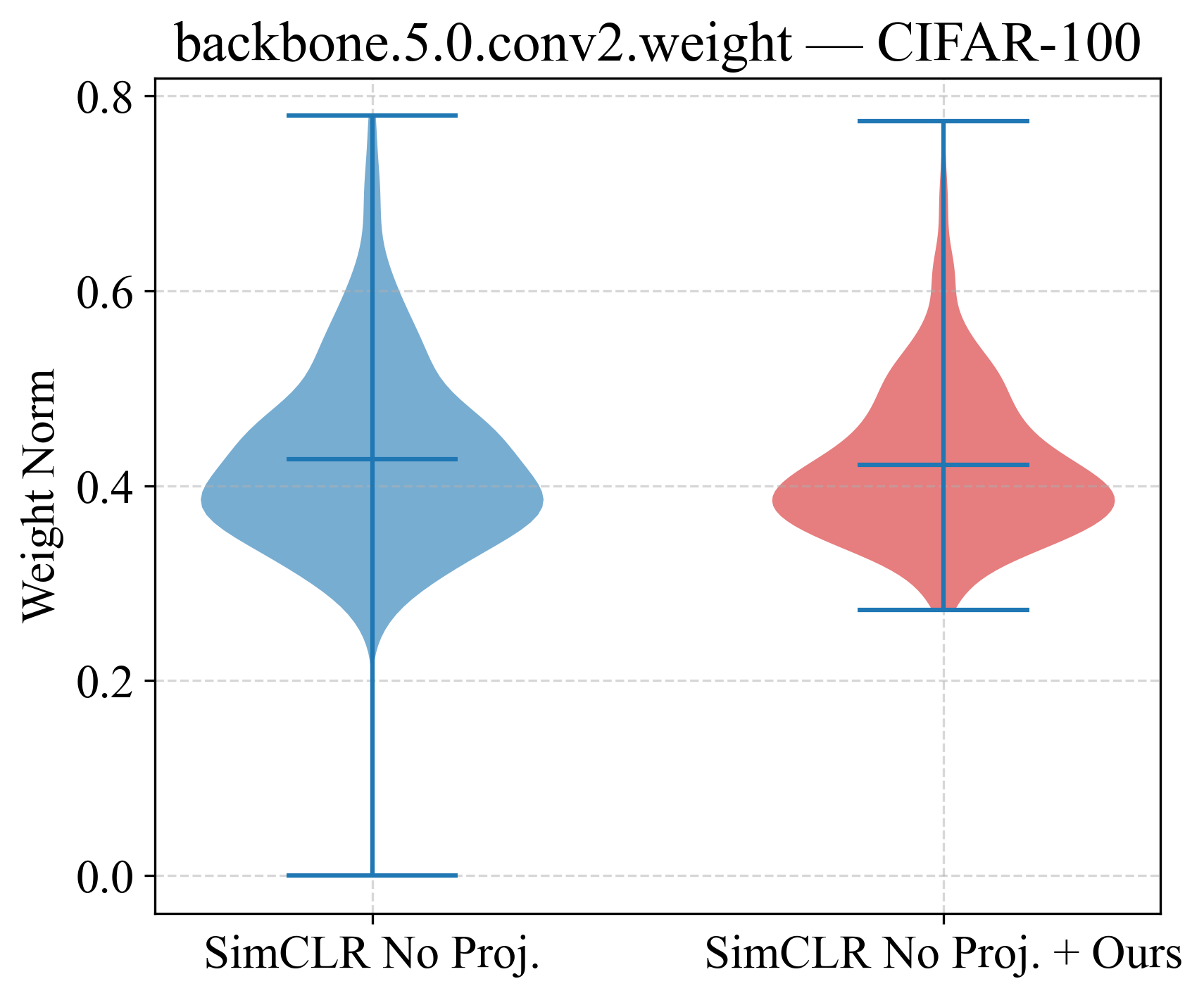}
      \caption{CIFAR100 - Last Layer - 1st Block - 2nd Conv}
      \label{fig:vio_proj_c100_502}
  \end{subfigure}
  \hfill
  \begin{subfigure}[b]{0.24\linewidth}
      \centering
    \includegraphics[width = 0.95\linewidth]{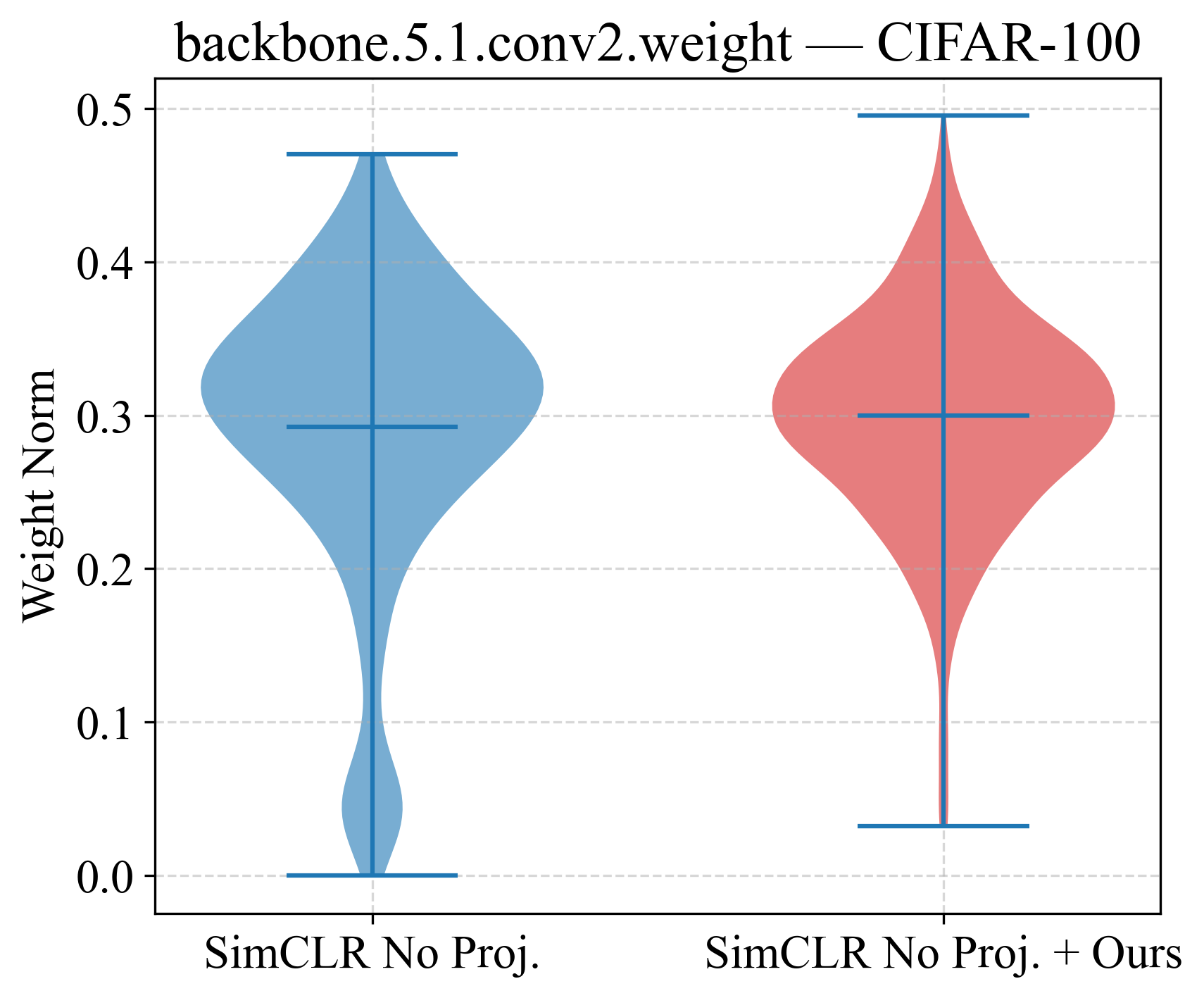}
      \caption{CIFAR100 - Last Layer - 2nd Block - 2nd Conv}
      \label{fig:vio_proj_c100_512}
  \end{subfigure}%
\caption{\textcolor{black}{Comparison of channel-wise weight norm distribution of SimCLR without projector trained without (blue) and with (red) our proposed weight regularization in the last layer for the penultimate and final layers of the projector. The distribution of weight norms of SimCLR without the projector shows that the weights of the second convolutional layer of the second ResNet block in the last layer \textbf{go very close to 0.0}. Whereas the proposed method can raise the minimum \textbf{away from 0.0} despite having similar mean values. Best viewed at 200\%.}}
  \label{fig:vio_woproj}
\end{figure}

\begin{figure}[!ht]
  \centering
  \begin{subfigure}[b]{0.24\linewidth}
      \centering
    \includegraphics[width = 0.95\linewidth]{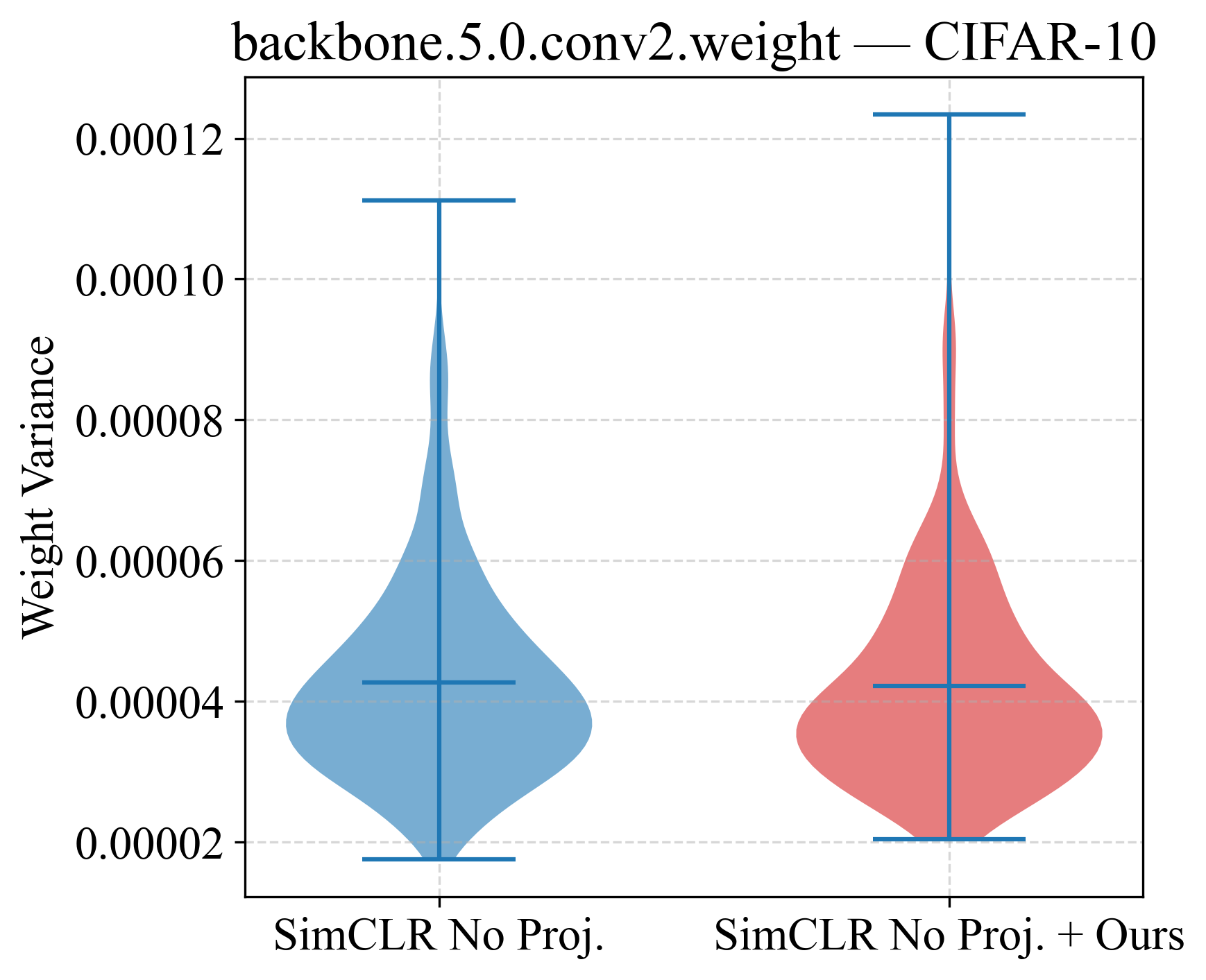}
      \caption{CIFAR10 - Last Layer - 1st Block - 2nd Conv}
      \label{fig:vio_var_proj_c10_502}
  \end{subfigure}
  \hfill
  \begin{subfigure}[b]{0.24\linewidth}
      \centering
    \includegraphics[width = 0.95\linewidth]{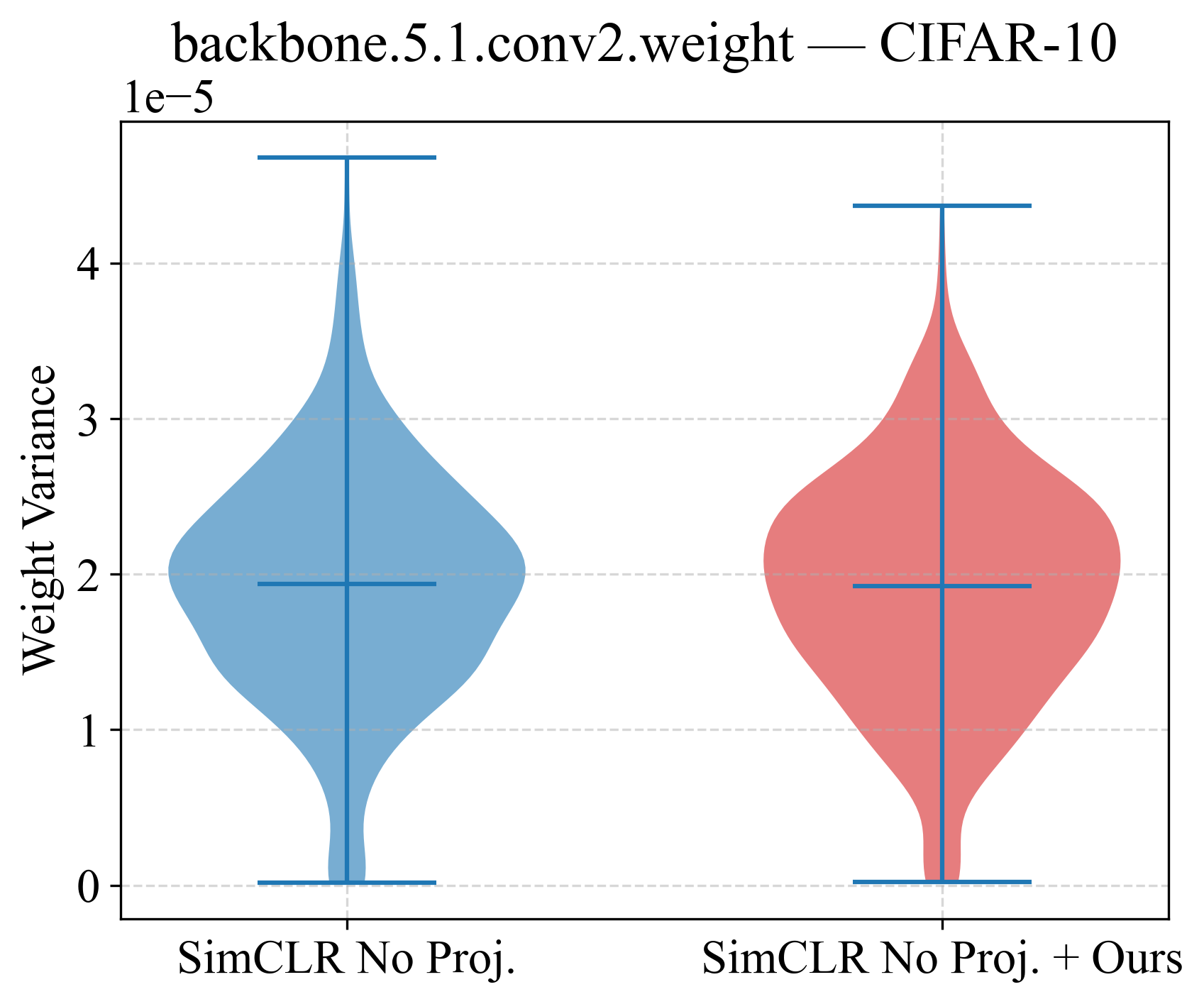}
      \caption{CIFAR10 - Last Layer - 2nd Block - 2nd Conv}
      \label{fig:vio_var_proj_c10_512}
  \end{subfigure}%
  \hfill
  \begin{subfigure}[b]{0.24\linewidth}
      \centering
    \includegraphics[width = 0.95\linewidth]{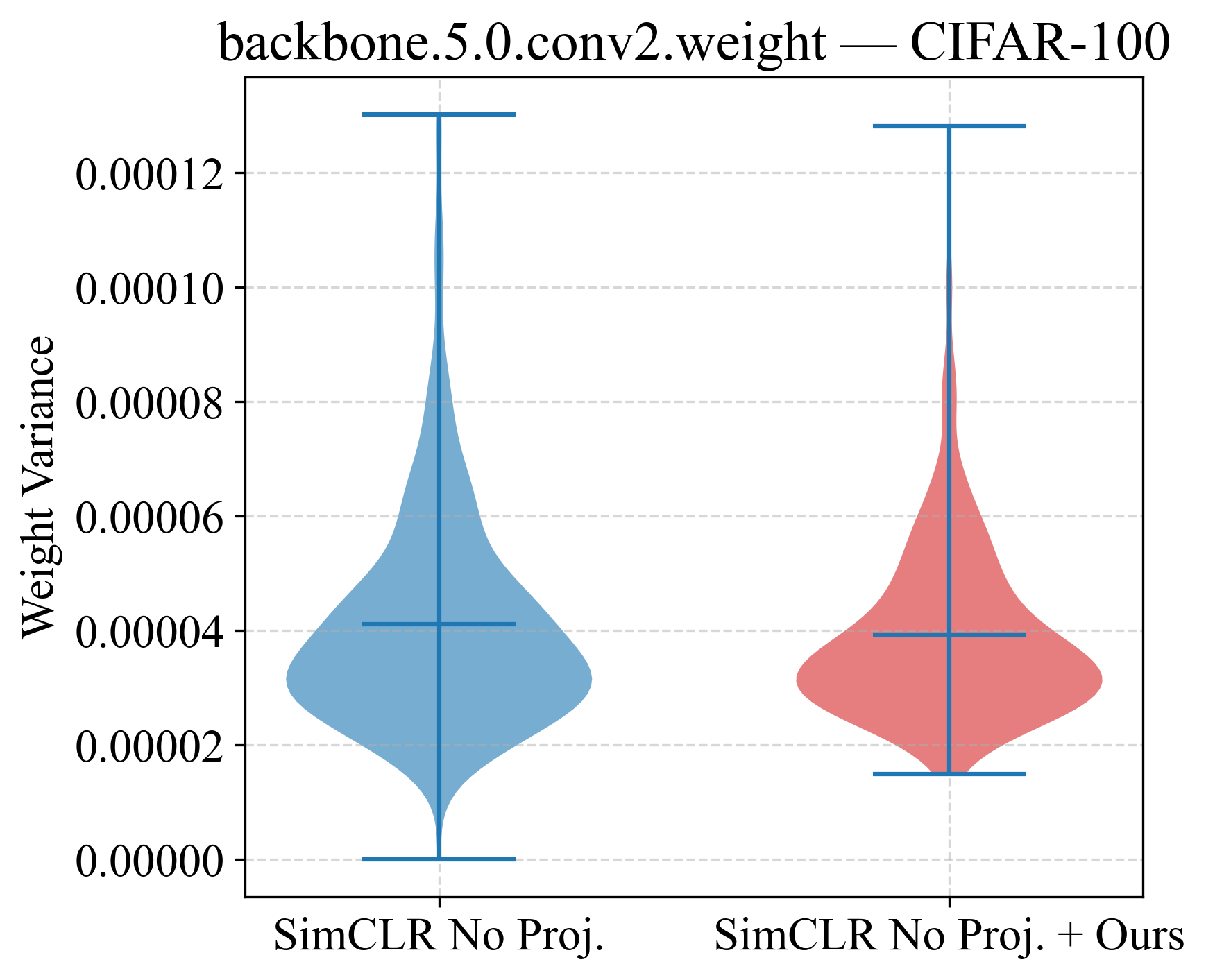}
      \caption{CIFAR100 - Last Layer - 1st Block - 2nd Conv}
      \label{fig:vio_var_proj_c100_502}
  \end{subfigure}
  \hfill
  \begin{subfigure}[b]{0.24\linewidth}
      \centering
    \includegraphics[width = 0.95\linewidth]{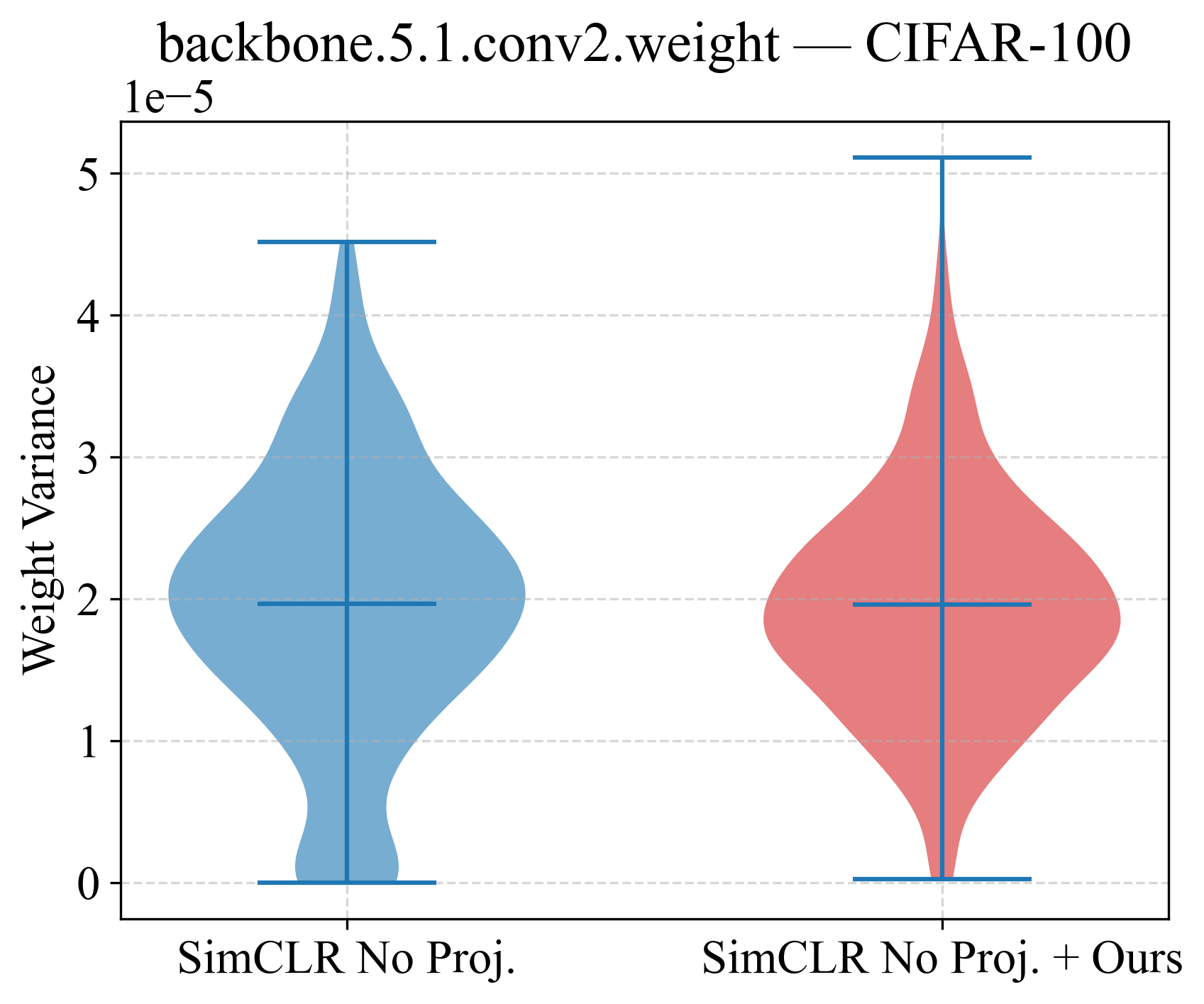}
      \caption{CIFAR100 - Last Layer - 2nd Block - 2nd Conv}
      \label{fig:vio_var_proj_c100_512}
  \end{subfigure}%
\caption{\textcolor{black}{Comparison of channel-wise unbiased variance of weight of SimCLR without projector trained without (blue) and with (red) our proposed weight regularization in the last layer for the penultimate and final layers of the projector. The distribution of unbiased variances of the weight of the SimCLR without the projector shows that the weights of the second convolutional layer of the second resnet block in the last layer \textbf{go very close to 0.0}. Whereas the proposed method can raise the minimum \textbf{away from 0.0} despite having similar mean values. Best viewed at 200\%.}}
  \label{fig:vio_var_woproj}
\end{figure}

\textcolor{black}{\paragraph{Empirical Justification:} We look at the distribution of the channel-wise weight norms of the convolutional layers in the penultimate and last layer of the encoder (Fig. \ref{fig:vio_woproj}), trained using the SimCLR framework, but without a projector. We can observe that there are several samples in the distributions which have values very close to zero, indicating the presence of collapsed dimensions. Additionally, we also present the effect of applying a weight regularization in the last layer only when trained without a projector in the same figure.}

\textcolor{black}{The importance of these empirical findings lies in the fact that they directly corroborate the ``Key Takeaway 1'' mentioned previously. By definition, the last layer corresponds to learning the high-level and more complex representations. Having a negligible weight norm and a negligible variance reduces the representation learning capacity of the layers in concern. From Figs. \mbox{\ref{fig:vio_woproj}} and \mbox{\ref{fig:pen_enc_c100}}, we can combine the empirical results to infer that in the absence of a projector, both the covariance of the embedding dimensions and the weight norm become negligible along a few dimensions. Hence, considering that the concerned weights are from the last layer, we can safely state that the kernels in the deeper layers are incapable of learning useful information in the absence of a projector.}

\subsubsection{Propagation of Low-Rank Representations}
\label{subsubsec:proplowrank}

In the case of a perfect decorrelating effect of the InfoNCE loss as discussed in the previous subsection, the flow of information would be maximized, resulting in better semantic representation learning and consequently better downstream performance. \textit{But the decorrelation effect of the InfoNCE on the encoder output embeddings does not maximize the flow of information through the encoder output dimensions}. A similar approach to prevent dimensional collapse was also presented in WeRank \citep{pasand2024werank}, where the authors used weight regularization by whitening the weight covariance matrix.

\paragraph{\textbf{Empirical justification and observations}}: We study the singular value plots of the penultimate and final layer of the encoder backbone in Fig. \ref{fig:l4n5_wnwoproj}.
  We see an increase in the number of high-valued singular values in the final layer of the backbone when using a projector. However, the eigenvalue spectrum is almost unchanged with or without a projector in the penultimate layer of the ResNet backbone encoder. \textit{Thus, we can say that the effect of rank reduction is observed only on the final layer, when not using a separate projector, and the final layer of the backbone acts as the make-shift or pseudo projector}. However, according to our observation, it is wise to say that the rank reduction effect is observed on the layer from which the final embedding is taken for contrastive loss computation, while the eigenspectrum of the previous layer output embeddings shows almost no change. This is confirmed from the plots of eigenvalues in Fig. \ref{fig:pen_proj}, where we observe that the eigenvalues for the last layer of the projector are significantly low in magnitude.

 \begin{figure*}
      \centering
      \begin{subfigure}[b]{0.32\linewidth}
          \centering
        \includegraphics[width = 0.9\linewidth]{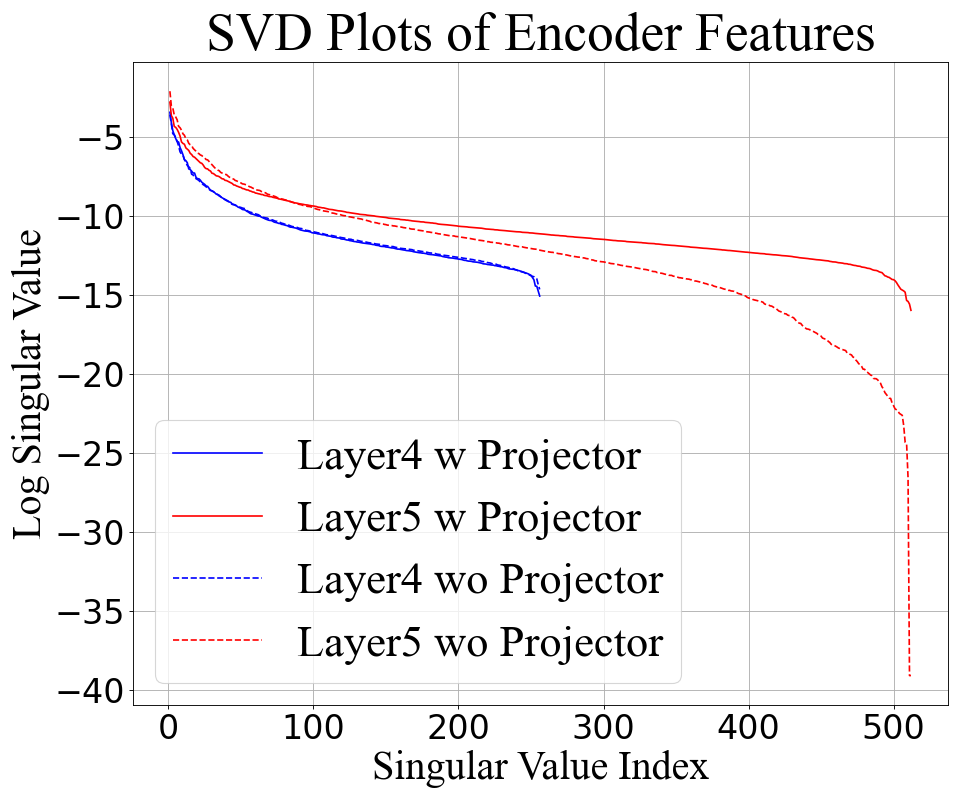}
          \caption{CIFAR10}
          \label{fig:pen_enc_c10}
      \end{subfigure}%
      \hfill
      \begin{subfigure}[b]{0.32\linewidth}
          \centering
        \includegraphics[width = 0.9\linewidth]{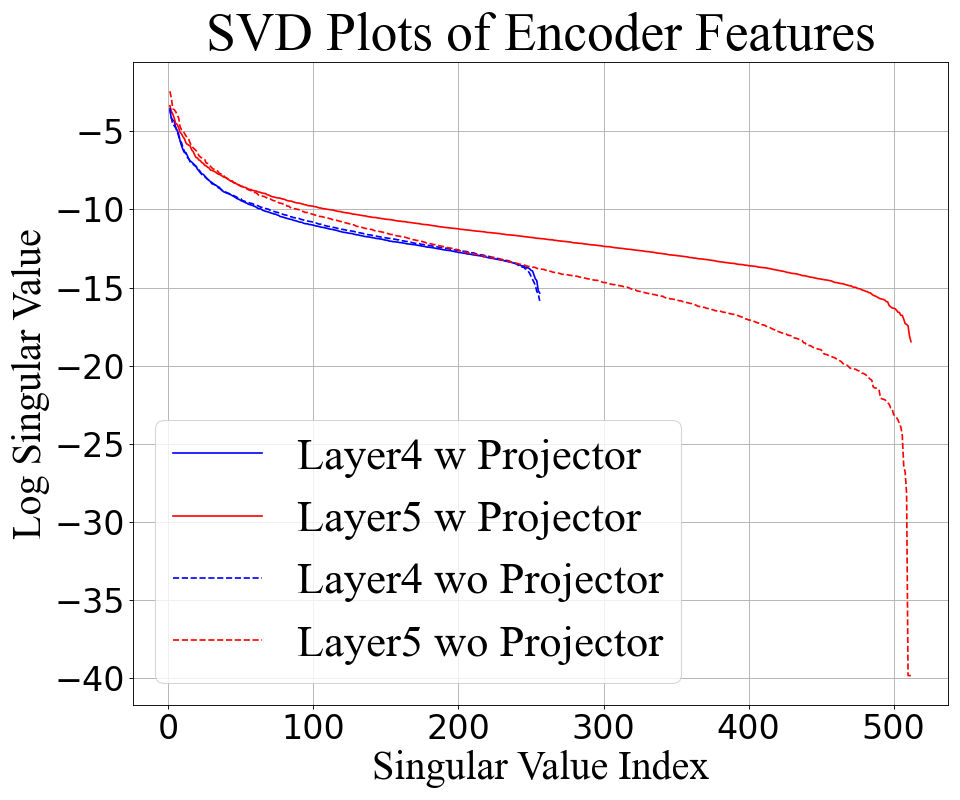}
          \caption{CIFAR100}
          \label{fig:pen_enc_c100}
      \end{subfigure}%
      \hfill
      \begin{subfigure}[b]{0.32\linewidth}
          \centering
        \includegraphics[width = 0.9\linewidth]{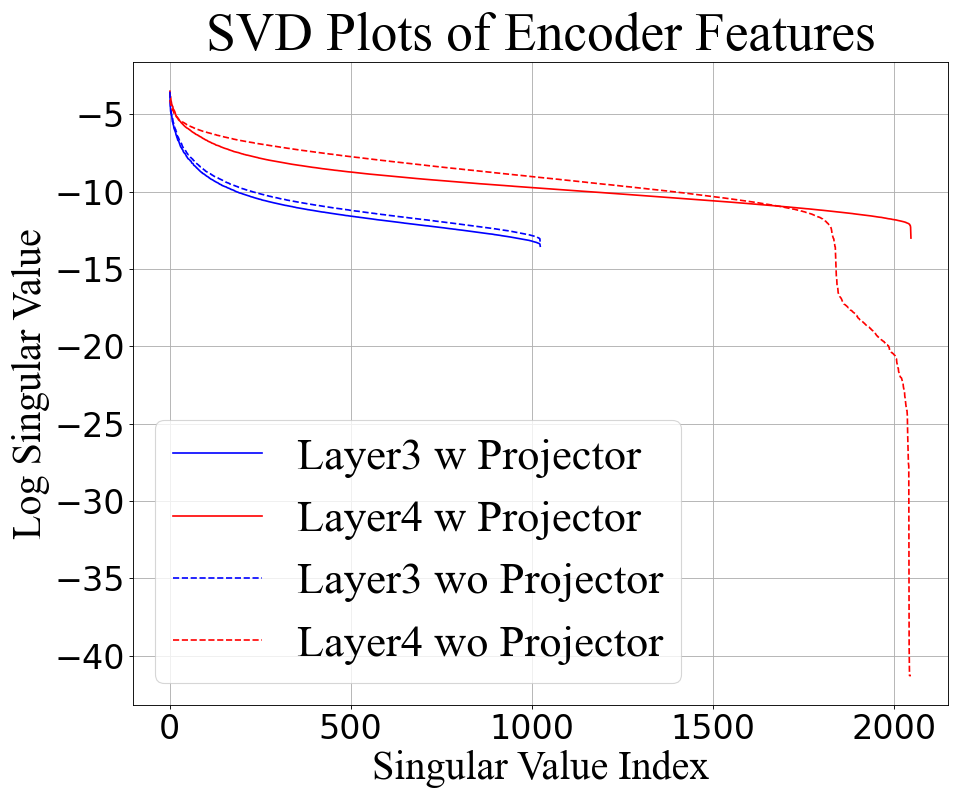}
          \caption{ImageNet100}
          \label{fig:pen_enc_in100}
      \end{subfigure}
    \caption{Singular values plots of the penultimate and final layers of the encoder with and without the projector. \textcolor{black}{The plots above show the effect of the presence and absence of the projector on the last and the penultimate layer of the encoder. The last layer (Layer5) shows a clear dimensional collapse without the projector in all three datasets, CIFAR10, CIFAR100 and ImageNet100. However, there is a negligible change in the singular value plot for the penultimate layer (Layer4), which agrees with our previous observation that the decorrelation effect is not fully propagated towards the shallower layers.} Best viewed at 300\%.}
      \label{fig:l4n5_wnwoproj}
  \end{figure*}

   \begin{figure*}[h!]
      \centering
      \begin{subfigure}[b]{0.32\linewidth}
          \centering
        \includegraphics[width = 0.9\linewidth]{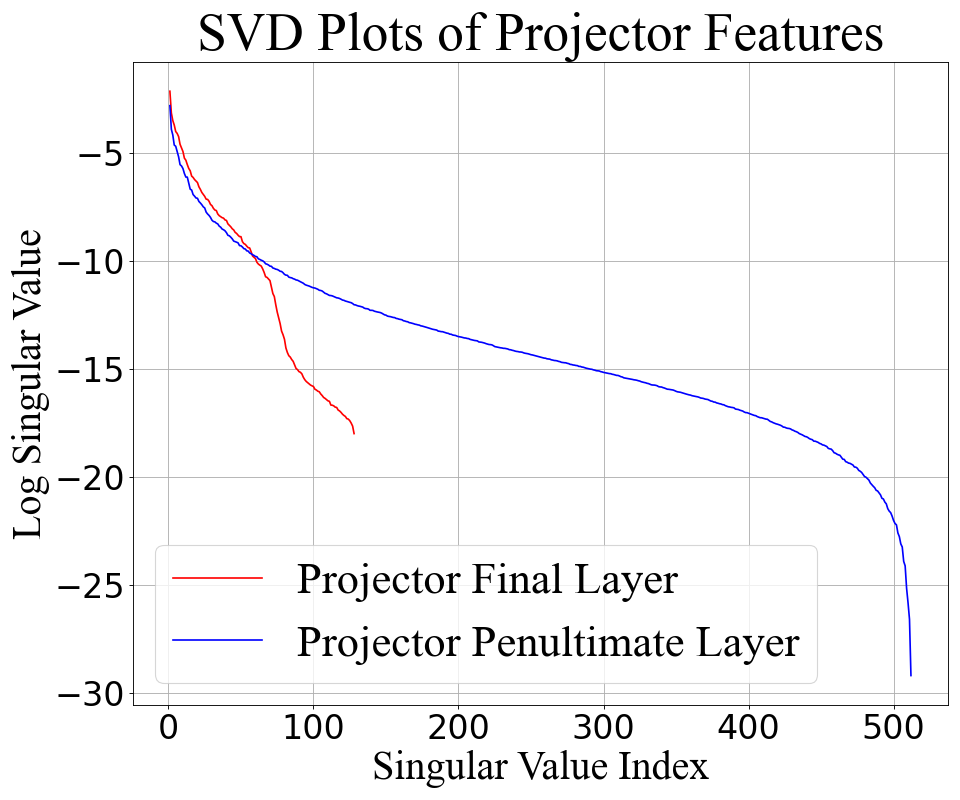}
          \caption{CIFAR10}
          \label{fig:pen_proj_c10}
      \end{subfigure}
      \hfill
      \begin{subfigure}[b]{0.32\linewidth}
          \centering
        \includegraphics[width = 0.9\linewidth]{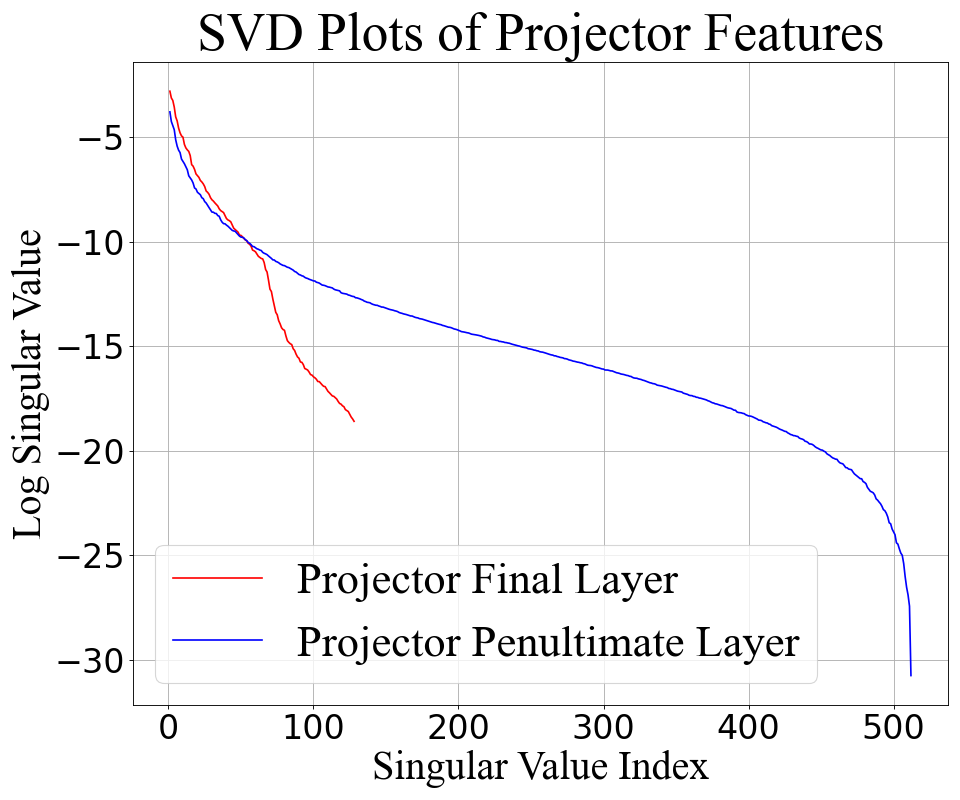}
          \caption{CIFAR100}
          \label{fig:pen_proj_c100}
      \end{subfigure}
      \hfill
      \begin{subfigure}[b]{0.32\linewidth}
          \centering
        \includegraphics[width = 0.9\linewidth]{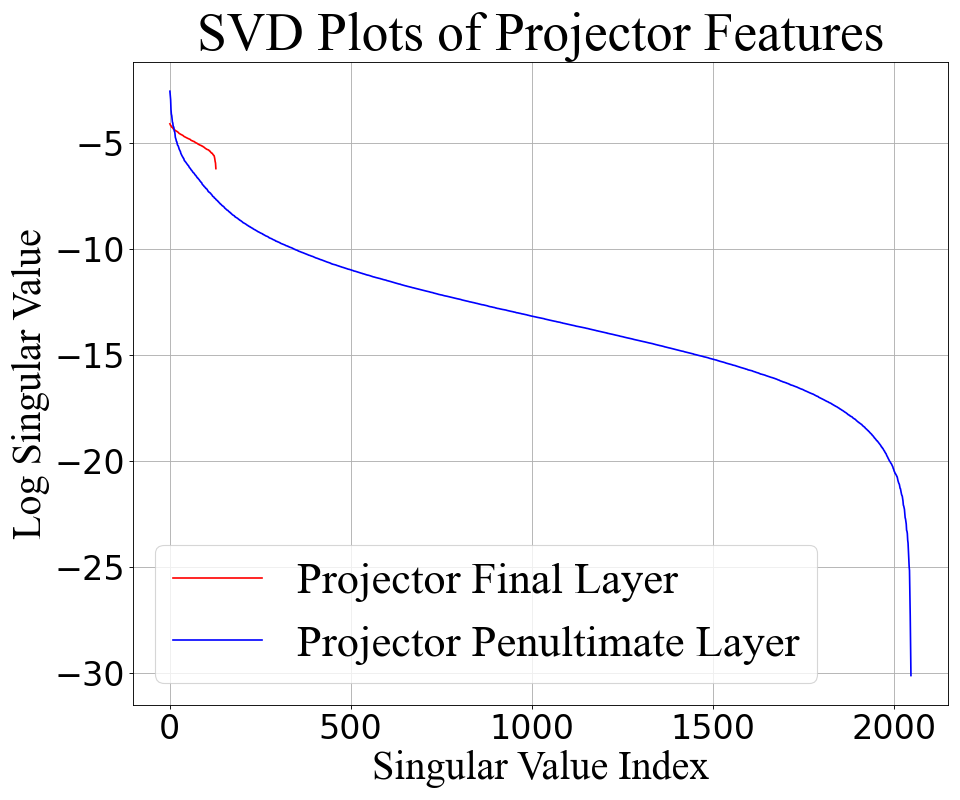}
          \caption{ImageNet100}
          \label{fig:pen_proj_in100}
      \end{subfigure}
    \caption{Singular values plots of the penultimate and final layers of the projector. Best viewed at 300\%.}
      \label{fig:pen_proj}
  \end{figure*}

\paragraph{\textbf{Why are full-rank embeddings better than low-rank embeddings?}} The trick used in DirectCLR is to leave a part of the output vector $z$ randomly initialized, theoretically making the variance of those dimensions non-zero. This causes the encoder to learn representations that are effectively higher-dimensional. Therefore, separability is also better according to \textit{Cover's theorem} \citep{coverstheorem}. Whereas, when the rank gets reduced, the representations are mapped to a low-dimensional subspace which \textit{should} reduce the probability that the mapped instances are linearly separable while still being embedded in a high-dimensional space. 

{A common premise in the literature, often motivated by frameworks like RankMe \citep{garrido2023rankme} is that higher-rank embeddings directly correlate with superior downstream performance. RankMe evaluates a within-method, cross-hyperparameter relationship between representation rank and linear evaluation performance within a single, consistent feature space. In contrast, evaluating pre-projector features $h \in \mathbb{R}^{D_h}$ versus post-projector features $z \in \mathbb{R}^{D_z}$ within the same trained network involves projections across spaces with fundamentally different ambient dimensions ($D_h \neq D_z$) and distinct functional roles.  The projection head intentionally maps $h$ into $z$ to isolate loss-specific decorrelation, whereas $h$ retains a richer, multi-dimensional semantic hierarchy optimized for general linear separability. Because $h$ and $z$ reside in distinct ambient spaces, comparing their absolute ranks directly does not imply that $z$ is a superior representation overall. Instead, rank dynamics across the pre- and post-projector boundary reflect geometric transformations between distinct feature spaces rather than a naive "higher rank implies better representation" rule. }

{

\noindent \textbf{Proposition 2.} Let $x_{l-1} \in \mathbb{R}^{D_{l-1}}$ be a random vector with covariance $\Sigma_{l-1} = \mathrm{Cov}(x_{l-1}) \succ 0$. Let $x_l = W_{l-1} x_{l-1} \in \mathbb{R}^{D_l}$ have covariance $\Sigma_{x_l} = W_{l-1} \Sigma_{l-1} W_{l-1}^\top$ with spectral decomposition $\Sigma_{x_l} = \sum_{k=1}^{D_l} \lambda_k \mathbf{v}_k \mathbf{v}_k^\top$, where $\lambda_1 \ge \dots \ge \lambda_{D_l} > 0$. Let $h = W_l x_l \in \mathbb{R}^{D_o}$ with output covariance $\Sigma_h = \operatorname{Cov}(h) = W_l \Sigma_{x_l} W_l^\top \in \mathbb{R}^{D_o \times D_o}$.

Suppose the $i$-th eigenvalue of $\Sigma_h$ exhibits practical collapse, i.e., $\lambda_i(\Sigma_h) \le \varepsilon$. Then:
\begin{itemize}
    \item $\lambda_i(\Sigma_h) \le \varepsilon$ does not imply that $\operatorname{rank}(W_{l-1}) < D_l$ or that $\Sigma_{l-1}$ is ill-conditioned.
    \item In particular, for any well-conditioned input $\Sigma_{l-1} \succ 0$ and full-rank matrix $W_{l-1}$, $\lambda_i(\Sigma_h) \le \varepsilon$ occurs if and only if the projected eigenvector $\mathbf{w}_i = W_l^\top \mathbf{u}_i$ satisfies:
    \begin{equation}
        \sum_{k=1}^{D_l} \lambda_k \left( \mathbf{u}_i^\top W_l \mathbf{v}_k \right)^2 \le \varepsilon
    \end{equation}
    where $\mathbf{u}_i$ is the $i$-th orthonormal eigenvector of $\Sigma_h$.
\end{itemize}

\noindent \textbf{Proof:} 
We assume that the network under consideration is consisted of two consecutive layers in a linear network without skip connections or activation functions:
\begin{equation}
    x_l = W_{l-1} x_{l-1}, \quad h = W_l x_l = W_l W_{l-1} x_{l-1}
\end{equation}
where $x_{l-1} \in \mathbb{R}^{D_{l-1}}$, $x_l \in \mathbb{R}^{D_l}$, and $h \in \mathbb{R}^{D_o}$ are column vectors. Let $\Sigma_{l-1} = \mathbb{E}\left[(x_{l-1} - \mathbb{E}[x_{l-1}])(x_{l-1} - \mathbb{E}[x_{l-1}])^T\right] \in \mathbb{R}^{D_{l-1} \times D_{l-1}}$ denote the covariance matrix of $x_{l-1}$. The representation covariance matrix $\Sigma_h = \mathrm{Cov}(h) \in \mathbb{R}^{D_o \times D_o}$ is given by:
\begin{equation}
    \Sigma_h = W_l W_{l-1} \Sigma_{l-1} W_{l-1}^T W_l^T
    \label{eqn:cov_h_spectral}
\end{equation}

Let $\Sigma_h = U \Lambda U^T$ be the spectral decomposition of $\Sigma_h$, where $U = [\mathbf{u}_1, \dots, \mathbf{u}_{D_o}] \in \mathbb{R}^{D_o \times D_o}$ is an orthogonal matrix of eigenvectors ($U^T U = I$) and $\Lambda = \mathrm{diag}(\lambda_1, \dots, \lambda_{D_o})$ contains the output eigenvalues ordered as $\lambda_1 \ge \lambda_2 \ge \dots \ge \lambda_{D_o} \ge 0$.

By Rayleigh-Ritz \citep{Trefethen2022NLA}, the $i$-th eigenvalue $\lambda_i(\Sigma_h)$ corresponds to the variance along the $i$-th principal component direction $\mathbf{u}_i$:
\begin{equation}
    \lambda_i(\Sigma_h) = \mathbf{u}_i^T \Sigma_h \mathbf{u}_i = \mathbf{u}_i^T \left( W_l W_{l-1} \Sigma_{l-1} W_{l-1}^T W_l^T \right) \mathbf{u}_i = \left\| \mathbf{u}_i^T W_l W_{l-1} \Sigma_{l-1}^{\frac{1}{2}} \right\|_2^2
    \label{eqn:lambdalcovh}
\end{equation}

We can safely state that, practical dimensional collapse along the $i$-th principal component occurs when $\lambda_i(\Sigma_h) \le \epsilon$ for a small threshold $\epsilon > 0$. This yields:
\begin{equation}
    \left\| \mathbf{u}_i^T W_l W_{l-1} \Sigma_{l-1}^{\frac{1}{2}} \right\|_2 \le \sqrt{\epsilon} = \epsilon'
    \label{eqn:spectral_collapse_bound}
\end{equation}

The above Eqn. \ref{eqn:spectral_collapse_bound} reveals two distinct structural mechanisms through which collapse occurs:

\noindent
\textbf{Case 1 (Spectral Misalignment with Feature Covariance):}
Let $\Sigma_{l-1} = V \Lambda_{l-1} V^T = \sum_{k=1}^{D_{l-1}} \mu_k \mathbf{v}_k \mathbf{v}_k^T$ be the spectral decomposition of $\Sigma_{l-1}$, where $\mu_k \ge 0$ are the input eigenvalues and $\mathbf{v}_k$ are the orthonormal eigenvectors. From Eqn. \ref{eqn:lambdalcovh}, we expand the expression for $\lambda_i(\Sigma_h)$ yielding:
\begin{equation}
    \lambda_i(\Sigma_h) = \sum_{k=1}^{D_{l-1}} \mu_k \left( \mathbf{u}_i^T W_l W_{l-1} \mathbf{v}_k \right)^2 \le \epsilon
\end{equation}
If $\lambda_i(\Sigma_h) \le \epsilon$, the projected vector $\mathbf{u}_i^T W_l W_{l-1}$ lies in the approximate left nullspace of $\Sigma_{l-1}^{\frac{1}{2}}$, denoted as $\mathbf{u}_i^T W_l W_{l-1} \in \mathbf{\text{Null}}_{\epsilon'}(\Sigma_{l-1}^{\frac{1}{2}})$. This occurs when $\mathbf{u}_i^T W_l W_{l-1}$ is orthogonally misaligned with all dominant eigenvectors $\mathbf{v}_k$ corresponding to large eigenvalues $\mu_k \gg \epsilon$.

\noindent
\textbf{Case 2 (Inter-Layer Weight Misalignment or Shrinkage):}
Applying the Rayleigh-Ritz bounds \citep{Trefethen2022NLA} with respect to the spectrum of $\Sigma_{l-1}$:
\begin{equation}
    \mu_{\min}(\Sigma_{l-1}) \left\| \mathbf{u}_i^T W_l W_{l-1} \right\|_2^2 \le \lambda_i(\Sigma_h) \le \mu_{\max}(\Sigma_{l-1}) \left\| \mathbf{u}_i^T W_l W_{l-1} \right\|_2^2
\end{equation}
When $\Sigma_{l-1}$ is well-conditioned ($\mu_{\min}(\Sigma_{l-1}) > 0$), $\lambda_i(\Sigma_h) \le \epsilon$ forces the effective weight norm to be bounded:
\begin{equation}
    \left\| \mathbf{u}_i^T W_l W_{l-1} \right\|_2 \le \sqrt{\frac{\epsilon}{\mu_{\min}(\Sigma_{l-1})}} = \epsilon_l
\end{equation}
This establishes that $\mathbf{u}_i^T W_l$ lies in the approximate left nullspace of $W_{l-1}$, i.e., $\mathbf{u}_i^T W_l \in \mathbf{\text{Null}}_{\epsilon_l}(W_{l-1})$.

\noindent
From these spectral derivations, dimensional collapse ($\lambda_i(\Sigma_h) \le \epsilon$) originates from three non-unique mechanisms:
\begin{enumerate}
    \item \textbf{Inter-Layer Weight Shrinkage:} $\|\mathbf{u}_i^T W_l\|_2$ is small, that is, vanishing weight norm along output eigenvector $\mathbf{u}_i$.
    \item \textbf{Inter-Layer Spectral Misalignment:} $\mathbf{u}_i^T W_l$ aligns with the left approximate nullspace of $W_{l-1}$.
    \item \textbf{Feature Space Misalignment:} $\mathbf{u}_i^T W_l W_{l-1}$ aligns with the low-variance eigenspace of $\Sigma_{l-1}$.
\end{enumerate}

\paragraph{Key Takeaway 2: Dimensional collapse does not necessarily imply rank-deficient representations or weights in shallower layers.} 
A collapsed principal component mode of the representation at layer $l$, characterized by a near-zero output eigenvalue ($\lambda_i(\Sigma_h) \le \epsilon$), arises from multiple, non-exclusive spectral mechanisms rather than a single structural bottleneck. First, even when the input covariance $\Sigma_{l-1}$ and intermediate weights $W_{l-1}$ are strictly full-rank and well-conditioned, collapse along an output principal direction $\mathbf{u}_i$ can occur due to inter-layer spectral misalignment, where the projected weight vector $\mathbf{u}_i^T W_l$ falls into the approximate left nullspace of $W_{l-1}$. Second, collapse occurs when the composite transformation $\mathbf{u}_i^T W_l W_{l-1}$ is spectrally misaligned with the dominant eigenspaces of the input covariance $\Sigma_{l-1}$, or when $W_l$ undergoes directional weight shrinkage along $\mathbf{u}_i$. Consequently, observing a low-rank or collapsed representation spectrum at layer $l$ does not automatically imply rank deficiency in shallower weight matrices ($W_{l-1}$) or input embeddings. This finding refutes the premise in WeRank \citep{pasand2024werank}, which attributes collapse primarily to low-rank weights in earlier layers, demonstrating instead that representation collapse is an unidentifiable phenomenon co-driven by inter-layer spectral misalignment and feature covariance interactions.

\paragraph{Mechanisms of Representation Collapse.} 
Prior work, like WeRank \citep{pasand2024werank}, attribute output representation collapse primarily to rank-deficient or low-rank weight matrices in earlier layers. However, our 4-cause taxonomy demonstrates that observing a collapsed representation spectrum ($\lambda_i(\Sigma_h) \le \epsilon$) at layer $l$ is fundamentally an \emph{unidentifiable} phenomenon that cannot be uniquely mapped to shallow-layer rank deficiency. Specifically, while output collapse can indeed stem from weight magnitude shrinkage (Cause 1) or data-inherited degeneracy (Cause 4), it also routinely occurs when earlier weight matrices $W_{l-1}$ and input feature covariances $\Sigma_{l-1}$ are strictly full-rank and well-conditioned. In this regime, collapse is independently driven by geometric alignment mechanisms: inter-layer spectral misalignment, where the active layer projects onto the approximate left nullspace of $W_{l-1}$ (Cause 2), or feature-space misalignment, where the composite transformation rotates orthogonally to the dominant eigenspace of $\Sigma_{l-1}$ (Cause 3). By failing to disentangle geometric directional alignment from matrix rank, prior frameworks oversimplify the diagnostics of collapse; our theoretical results prove that output spectral collapse provides no necessary condition on the rank or conditioning of earlier-layer weights.
}

\subsubsection{Does stopping information flow along some dimensions of the encoder output result in a better information bottleneck than using a Projector?}

We attempt to substantiate our aforementioned statements more concretely and propose a straightforward approach to enhance performance on self-supervised contrastive learning tasks without requiring a projector. According to Property 5 described in \cite{fang2024rethinking}, a subvector of fixed value is the same as dimensional collapse along the dimensions of the subvector.  We intend to stop the information flow, resulting in an enforced dimensional collapse to cause an information bottleneck without the projector by fixing the output of the embeddings along a few dimensions to a constant $k$. First, let us go through the notations. We only keep $d_0$ out of a total of $D$ dimensions as \textit{dynamic}, while making the rest ($ D-d_0 = d_r$) \textit{static} by assigning a constant output value to those dimensions of the embedding.
Making the $i$-th dimension of the encoder output  $h$, that is, $h_i$, static with a constant $k = 0$, stops the gradient flow through all paths connected directly or indirectly to $h_i$, and the rank of the covariance matrix $\mathcal{C}$ follows Eqn. \ref{eqn:rankb}. However, the weights $W^i$, which result in $h_i$, are still randomly initialized. Whereas, in DirectCLR, due to the randomized subvector, the rank of the covariance matrix does not collapse drastically (Eqn. \ref{eqn:ranka}).

\begin{minipage}{0.48\linewidth}
\begin{equation}
    \text{rank}(\mathcal{C}) \le d_0
    \label{eqn:rankb}
\end{equation}
\end{minipage}
\hfill
\begin{minipage}{0.48\linewidth}
\begin{equation}
    d_r \le \text{rank}(\mathcal{C}) \le d_r + d_0
    \label{eqn:ranka}
\end{equation}
\end{minipage}

\noindent

Thus, when a constant value is not assigned to the \textit{static} $d_r$ dimensions, the rank of the encoder output embedding $h$ or, consequently, the random variable $\mathcal{H}$ is less than when the \textit{static} $d_r$ dimensions are left untouched.



Increasing the value of $d_0$ reduces the explicit dimensional collapse enforced on the representation space by allowing the rank of the embedding covariance matrix to increase according to Eqn. \ref{eqn:rankb}. Let us denote the new value of $d_0$ and $d_r$ be indicated by $d'_0$ and $d'_r$, where $d'_0 > d_0$ and $d'_r < d_r$. The rank of the new embedding covariance matrix also increases. So, we may think that we are effectively increasing the shattering capacity by (a) increasing the rank of the embedding covariance matrix, (b) decreasing the degree of dimensional collapse \citep{fang2024rethinking}, and also (c) increasing the dimensionality of the representation learning subspace \citep{coverstheorem}. However, this is not the case as observed from Table \ref{tab:c10n100_incrank}. The gradient of the InfoNCE loss $\mathcal{L}$ with respect to an embedding $z_i$ is given by Eqn. \ref{eqn:difflczj}. 

\begin{minipage}{\linewidth}
\begin{equation}
\label{eqn:difflczj}
\scriptsize
    \centering
    \begin{split}
        \frac{\partial \mathcal{L}}{\partial z_i} & =  \left[ - \frac{z_{i^+}}{\tau} + \frac{\frac{z_{i^+}}{\tau} \cdot e^{s_{ii^+}} + \sum_{\substack{j=1\\j \neq i}}^{N} \frac{z_j}{\tau} \cdot e^{s_{ij}}}{e^{s_{ii^+}} + \sum_{\substack{j=1\\j \neq i}}^{N} e^{s_{ij}}} + \sum_{\substack{j=1\\j \neq i}}^{N} \frac{\frac{z_j}{\tau} \cdot e^{s_{ji}}}{e^{s_{jj^+}} + \sum_{\substack{k=1\\k \neq j}}^{N} e^{s_{jk}}}\right]
        = - \left[ \frac{z_{i^+}}{\tau} \left( 1 - p^{ii^+} \right) - \sum_{\substack{j=1\\ j \neq i}}^{N} \frac{z_j}{\tau} \left( p^{ij} + p^{ji}\right) \right]\\
    \end{split}
\end{equation}
\end{minipage}

\noindent The quantity $p^{ij}$ is the probability of the pair $(x_i, x_j)$ being predicted as a positive pair with the sample $x_i$ as the anchor.  The gradient along each dimension can be written as follows, 

\begin{equation}
\label{eqn:lossDimGrad}
    \frac{\partial \mathcal{L}}{\partial z_i^d} = 
    \begin{cases}
   - \frac{k}{\tau} \left[  \left( 1 - p^{ii^+} \right) - \sum_{\substack{j=1\\ j \neq i}}^{N}  \left( p^{ij} + p^{ji}\right) \right] \;\;\;\;\;\;\;\;\;\;\; \text{for }d_0 < d \leq d_0+d_r\\    
  - \left[ \frac{z_{i+}^{d}}{\tau} \left( 1 - p^{ii^+} \right) - \sum_{\substack{j=1\\ j \neq i}}^{N} \frac{z_j^d}{\tau} \left( p^{ij} + p^{ji}\right) \right] \;\; \text{for }d \leq d_0\\  
\end{cases}
\end{equation}

Therefore, $\frac{\partial \mathcal{L}}{\partial z_i^d} = 0$ if $k = 0$ for the sub-vector $h[d_0:d_0+d_r]$, whereas the gradient flows normally through the rest of the dimensions. Using a constant other than 0 causes a small gradient to flow through all $d_r$ dimensions. This gradient disrupts proper training of the kernel weights, since the gradient of the $d_r$ sub-vector $(h[d_0:d_0+d_r])$ points toward $\mathbbm{1}_{d_r}$, which will eventually lead to dimensional collapse or may lead all points to lie within an open ball of finite radius along each of the $d_r$ dimensions. The performance, in this case, is worse than training with zero value in the $d_r$ dimensions and did not contribute towards maximizing the flow of information through the $d_0$ dimensions, as in DirectCLR \citep{jing2022directclr} or our framework with $k = 0$, due to the injection of a non-converging gradient. Empirical results for the CIFAR datasets are provided in Table \ref{tab:c10n100_incrank}.

\begin{minipage}{\linewidth}
    \centering
    \captionof{table}{200-NN Accuracy for different values of $d_0$ and $d_r$ on CIFAR10 and CIFAR100 dataset.}
    \vspace{0.5em}
    \begin{tabular}{c|c|c|c|c}
    \toprule
    \multirow{2}{*}{$d_0$} & \multirow{2}{*}{$d_r$} & \multirow{2}{*}{Fixed Value} & CIFAR10 & CIFAR100 \\\cline{4-5}
         & & & \multicolumn{2}{c}{200-NN Acc.} \\\midrule
        200 & 312 & 0 & 83.20 & 49.4\\\midrule
        200 & 312 & $\frac{1}{\sqrt{512}}$ & 82.5 & 48.6\\\midrule
        200 & 312 & 1 & 78.9 & 41.1\\\midrule\midrule
        400 & 112 & 0 & 84.2 & 52.2\\\midrule
        400 & 112 & $\frac{1}{\sqrt{512}}$ & 84.0 & 50.9\\\midrule
        400 & 112 & 1 & 79.8 & 44.3\\\midrule\midrule
        480 & 32 & 0 & 84.7 &  52.5\\ \midrule
        480 & 32 & $\frac{1}{\sqrt{512}}$ & 84.4 & 51.8\\\midrule
        480 & 32 & 1 & 80.9 & 45.9\\ 
        \bottomrule
    \end{tabular}
    \label{tab:c10n100_incrank}
\end{minipage}%

\begin{figure*}[h!]
      \centering
      \begin{subfigure}[b]{0.24\linewidth}
          \centering
        \includegraphics[width = \linewidth]{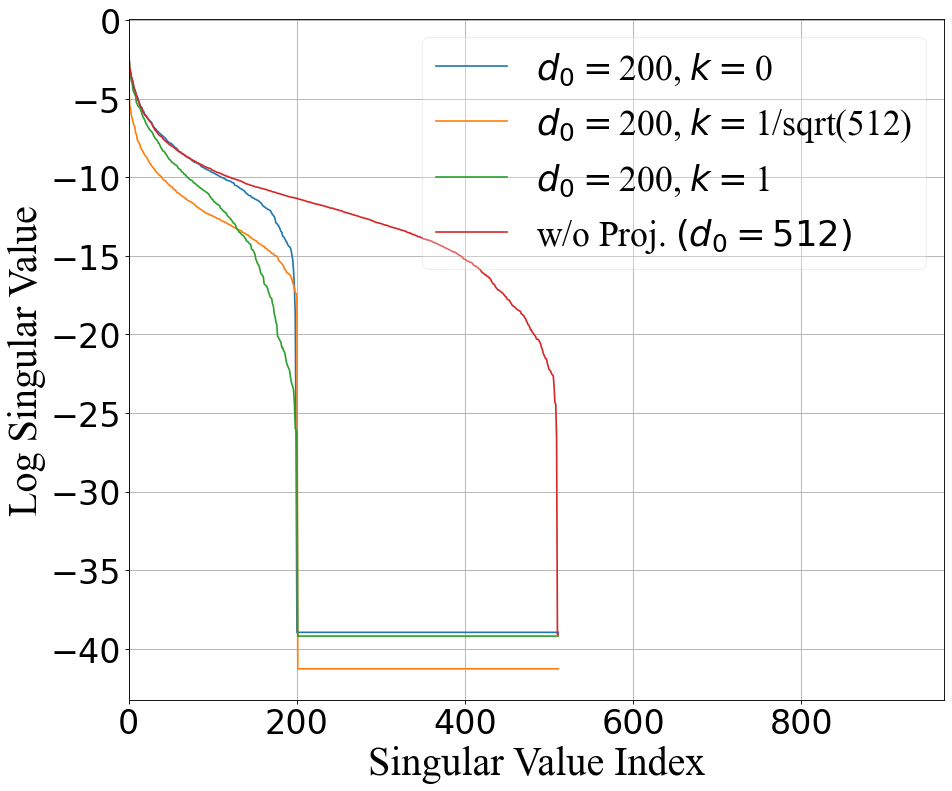}
          \caption{CIFAR10 ($d_0=200$)}
          \label{fig:200_c10}
      \end{subfigure}
      \hfill
      \begin{subfigure}[b]{0.24\linewidth}
          \centering
        \includegraphics[width = \linewidth]{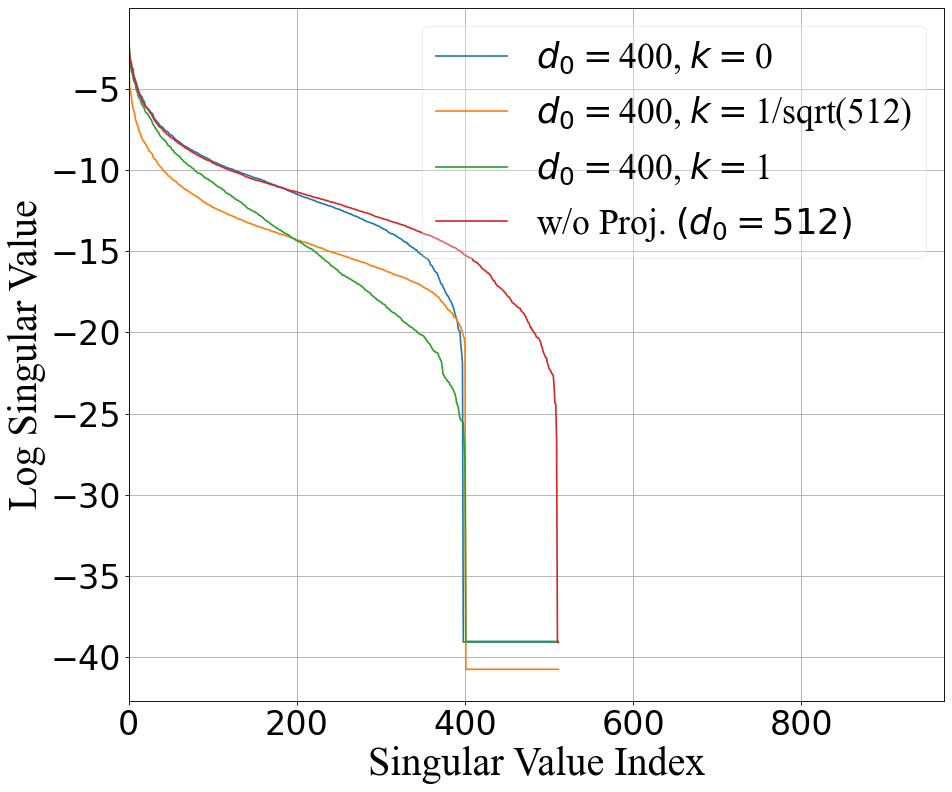}
          \caption{CIFAR10 ($d_0=400$)}
          \label{fig:400_c10}
      \end{subfigure}%
      \hfill
      \begin{subfigure}[b]{0.25\linewidth}
          \centering
        \includegraphics[width = 0.96\linewidth]{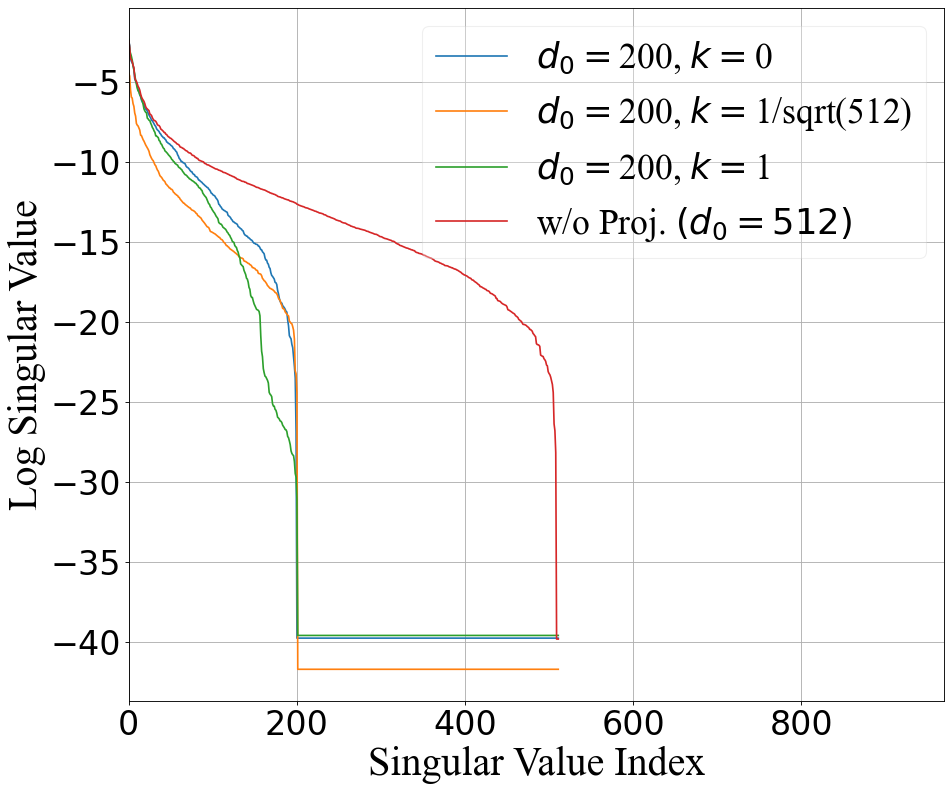}
          \caption{CIFAR100($d_0=200$)}
          \label{fig:200_c100}
      \end{subfigure}
      \hfill
      \begin{subfigure}[b]{0.25\linewidth}
          \centering
        \includegraphics[width = 0.96\linewidth]{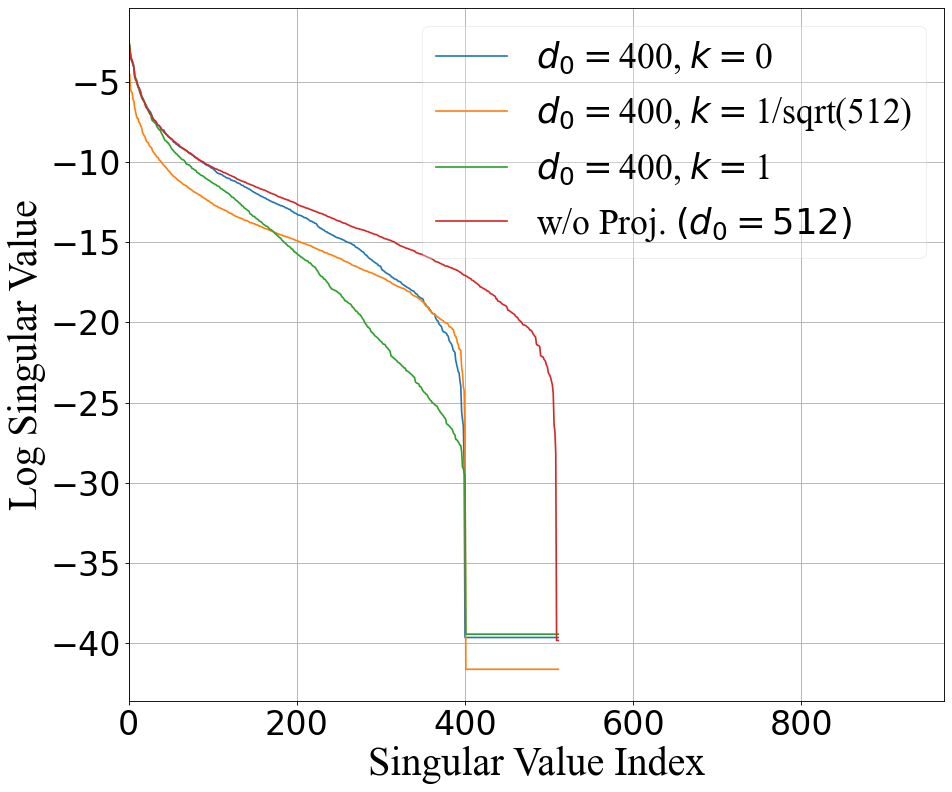}
          \caption{CIFAR100($d_0=400$)}
          \label{fig:400_c100}
      \end{subfigure}%
    \caption{Singular values plots of encoder outputs embeddings when using a constant subvector in the output embeddings for different values of $d_0$ on CIFAR10 and CIFAR100 datasets. \textcolor{black}{The above figures show the effect on the singular values of the encoder output embeddings, and it is evident that a forced collapse does not have the same effect as a natural dimensional collapse }. Best viewed at 300\%.}
      \label{fig:propgraphs}
  \end{figure*}

\paragraph{Key Takeaway 3: Forced Collapse does not enforce an Information Bottleneck} 

The empirical results provided in Tables \ref{tab:c10n100_incrank} combined with Fig. \ref{fig:propgraphs} indicate that we cannot observe the same information bottleneck effect without a projector. From the discussion in this subsection, we can safely say that the role of the projector is not only that of an information bottleneck, which is observed from the ineffectiveness of the forced collapse on the encoder output. The driving factor behind the effectiveness of the projector is that it allows the encoder to learn more high-level representations and, consequently, better separability with a higher rank of the embedding covariance matrix. Whereas a collapse in the last layer of the encoder prevents it from learning useful representations, which are essential for the effective classification of the input samples, as the norm of several kernels will be reduced to zero. With a constant subvector in the output, the weights are not updated to learn useful representations. Furthermore, being unable to learn essential representations also diminishes the mutual information between the input and the output and the generalization error bound, which we prove in the next subsection.

\subsection{Proposed Method: Remedy based on the Takeaways}



{To address weight-induced dimensional collapse without distorting the learned feature geometry of shallower layers, we introduce an orthogonal weight regularization restricted exclusively to the final layer ($W_l$). Guided by Proposition 1 and Proposition 2, applying this regularization strictly to the projection head eliminates weight magnitude shrinkage (Cause 1) by enforcing row-orthogonality ($W_l W_l^\top = I$) and preserving the full-rank capacity ($\mathrm{rank}(W_l) = D_o$) of $W_l$.}


{Empirical observations in contrastive SSL reveal that the penultimate feature covariance $\Sigma_{l-1} = \mathrm{Cov}(x_{l-1})$ maintains an intrinsically ill-conditioned eigenspectrum ($\kappa(\Sigma_{l-1}) \gg 1$), reflecting the natural variance hierarchy of learned semantic representations. In this ill-conditioned regime (Cause 4), the output representation is hyper-sensitive to weight decay and rank degradation in the projection head $W_l$. Applying orthogonal regularization ($W_l W_l^\top = I$) exclusively to the last layer eliminates weight-induced collapse (Cause 1) and guarantees that $W_l$ maintains full rank ($\mathrm{rank}(W_l) = D_o$), without distorting the functional feature geometry or semantic hierarchy of the backbone features $x_{l-1}$. Additionally, preventing the weight norm from collapsing preserves the representational capacity of the layer, thereby preventing artificial information loss between the penultimate and output representations.}

{\color{black}

{
\paragraph{Proposition 3: Weight regularization and mutual information:} Let $Z = W X \in \mathbb{R}^{D}$ be the output embedding used in the InfoNCE
objective, where $x$ is a random vector with $\Sigma_X=\mathrm{Cov}(X)\succ 0$, $Z$ is also a random vector such that $Z = [z_1, z_2, \hdots, z_{D_o}]$
and $W\in\mathbb{R}^{D_o\times D_i}$ is the weight matrix of the final layer.
Assume that $Z_{k}$ and $Z_{l}$ form a positive pair generated from two augmentations of the same input $X_k,X_l\in\mathbb{R}^{D_i}$, and let
\begin{equation}
Z_k=WX_k,\qquad Z_l=WX_l,
\end{equation}

Assume that $X_k$ and $X_l$ have finite second-order moments. We compare two cases: (i) an orthogonality-constrained representation satisfying $WW^\top=I$, and (ii) an unconstrained ({denoted by ``unc.'' from here onwards}) representation for which weight regularization may reduce the singular values of $W$ and consequently reduce its effective rank. Then the orthogonality constraint preserves the mutual information between the two augmented views, whereas an unconstrained $W$ can reduce the mutual information when it suppresses directions carrying information shared between the two views, that is,

\begin{minipage}{\linewidth}
\begin{equation}
\mathcal{I}(Z_{k};Z_{l} \mid WW^\top=I) \;\ge\; \mathcal{I}(Z_{k};Z_{l} \mid W_{\text{unc.}})
\end{equation}
\end{minipage}

\noindent \textbf{Proof:} Consider the final representation layer $Z = W X \in \mathbb{R}^{D_o}$, where $W \in \mathbb{R}^{D_o \times D_i}$ ($D_o \le D_i$) and $\Sigma_X = \operatorname{Cov}(X) \succ 0$. The output covariance is given by:
\begin{equation}
\Sigma_Z = \operatorname{Cov}(Z) = W \Sigma_X W^\top \in \mathbb{R}^{D_o \times D_o}
\end{equation}

Let $\mathbf{u}_i \in \mathbb{R}^{D_o}$ be an orthonormal eigenvector of $\Sigma_Z$ corresponding to eigenvalue $\lambda_i(\Sigma_Z)$, with $\|\mathbf{u}_i\|_2 = 1$. By the Rayleigh-Ritz theorem \citep{Trefethen2022NLA}:
\begin{equation}
\lambda_i(\Sigma_Z) = \mathbf{u}_i^\top \Sigma_Z \mathbf{u}_i = \mathbf{u}_i^\top \left( W \Sigma_X W^\top \right) \mathbf{u}_i = (W^\top \mathbf{u}_i)^\top \Sigma_X (W^\top \mathbf{u}_i)
\end{equation}

Since $\Sigma_X \succ 0$ is symmetric positive-definite, bounding the quadratic form by the extremal eigenvalues of $\Sigma_X$ yields:
\begin{equation}
\lambda_{\min}(\Sigma_X) \|W^\top \mathbf{u}_i\|_2^2 \;\le\; \lambda_i(\Sigma_Z) \;\le\; \lambda_{\max}(\Sigma_X) \|W^\top \mathbf{u}_i\|_2^2
\label{eqn:rayleigh_bounds}
\end{equation}
where $\|W^\top \mathbf{u}_i\|_2^2 = \mathbf{u}_i^\top W W^\top \mathbf{u}_i$. Equation \ref{eqn:rayleigh_bounds} establishes the exact connection between $WW^\top$ and output spectrum collapse:

\begin{enumerate}
    \item \textbf{Unconstrained Case ($WW^\top \neq I_{D_o}$):} Without orthogonal constraints, weight decay or rank deficiency permits $\|W^\top \mathbf{u}_i\|_2^2 \to 0$ along trailing modes $i \in \{r+1, \dots, D_o\}$. Using the upper bound in Eqn. \ref{eqn:rayleigh_bounds}, this directly forces output collapse:
    \begin{equation}
    \lambda_i(\Sigma_Z) \le \lambda_{\max}(\Sigma_X) \|W^\top \mathbf{u}_i\|_2^2 \to 0,
    \end{equation}
    The above equation indicates that the unconstrained case concentrates representations $Z$ onto an $r$-dimensional affine subspace $\mathcal{A}_Z \subset \mathbb{R}^{D_o}$ by the Eckart-Young-Mirsky theorem \citep{Humpherys2017AppliedMath}.
    
    \item \textbf{Orthogonally Constrained Case ($WW^\top = I_{D_o}$):} Enforcing $WW^\top = I_{D_o}$ guarantees that for any unit vector $\mathbf{u}_i$, $\|W^\top \mathbf{u}_i\|_2^2 = \mathbf{u}_i^\top I_{D_o} \mathbf{u}_i = 1$. Utilizing the lower bound in Eqn. \ref{eqn:rayleigh_bounds}, the output spectrum is strictly bounded away from zero:
    \begin{equation}
    \lambda_i(\Sigma_Z) \ge \lambda_{\min}(\Sigma_X) > 0 \quad \forall i \in \{1, \dots, D_o\},
    \end{equation}
    The above equation guarantees $\operatorname{rank}(\Sigma_Z) = D_o$ and preventing weight-induced dimensional collapse.
\end{enumerate}

\noindent \textbf{Preservation of Mutual Information via Orthogonal Regularization.} Next, we establish that enforcing $W$ to be orthogonal results in preservation of mutual information, compared to when the weights are unconstrained. We arrive at the above statement by progressing step-by-step through the proof as follows:

\paragraph{Step 1: Constructing the Matrix Factorization.}
Consider the output representations under the orthogonal reference matrix $W_{\mathrm{ortho}}$:
\begin{equation}
Z_k^{\mathrm{ortho}} = W_{\mathrm{ortho}} X_k, \qquad Z_l^{\mathrm{ortho}} = W_{\mathrm{ortho}} X_l \quad \in \mathbb{R}^{D_o}
\end{equation}
Also assume, there exists an orthogonally constrained final-layer weight matrix $W_{\mathrm{ortho}} \in \mathbb{R}^{D_o \times D_i}$ ($D_o \le D_i$) satisfying strict row-orthogonality: $W_{\mathrm{ortho}} W_{\mathrm{ortho}}^\top = I_{D_o}$. This guarantees that $W_{\mathrm{ortho}}$ is full row rank ($\operatorname{rank}(W_{\mathrm{ortho}}) = D_o$) and the matrix product $P_{\mathrm{ortho}} = W_{\mathrm{ortho}}^\top W_{\mathrm{ortho}} \in \mathbb{R}^{D_i \times D_i}$ forms an orthogonal projection operator onto $\operatorname{row}(W_{\mathrm{ortho}})$. 

Furthermore assume that, for any unconstrained weight matrix $W_{\text{unc.}} \in \mathbb{R}^{D_o \times D_i}$, its row space is contained within the $D_o$-dimensional row space of $W_{\mathrm{ortho}}$: $\operatorname{row}(W_{\text{unc.}}) \subseteq \operatorname{row}(W_{\mathrm{ortho}})$. 

Consequently, the matrix $P_{\mathrm{ortho}} = W_{\mathrm{ortho}}^\top W_{\mathrm{ortho}} \in \mathbb{R}^{D_i \times D_i}$ acts as an identity operator on the row space of $W_{\text{unc.}}$, yielding $W_{\text{unc.}} P_{\mathrm{ortho}} = W_{\text{unc.}}$, which implies $W_{\text{unc.}} P_{\mathrm{ortho}} = W_{\text{unc.}}$.  Next, we can define a transition matrix $M \in \mathbb{R}^{D_o \times D_o}$ as:
\begin{equation}
M = W_{\text{unc.}} W_{\mathrm{ortho}}^\top
\end{equation}

Next, if we multiply $M$ by $W_{\mathrm{ortho}}$ we get the following:
\begin{equation}
\begin{split}
    M W_{\mathrm{ortho}} &= \left(W_{\text{unc.}} W_{\mathrm{ortho}}^\top\right) W_{\mathrm{ortho}} \\&= W_{\text{unc.}} \left(W_{\mathrm{ortho}}^\top W_{\mathrm{ortho}}\right) = W_{\text{unc.}} P_{\mathrm{ortho}} = W_{\text{unc.}}
\end{split}
\end{equation}
Thus, $W_{\text{unc.}}$ cleanly decomposes into a linear transformation $M$ acting directly on $W_{\mathrm{ortho}}$:
\begin{equation}
W_{\text{unc.}} = M W_{\mathrm{ortho}}
\end{equation}

In the next step, we map the representations obtained using the unconstrained and orthogonal weights and formulate using a Markov chain to find the relationship between the mutual information of the two.

\paragraph{Step 2: Representation Mapping and Markov Chain Formulation.}
Using the factorization $W_{\text{unc.}} = M W_{\mathrm{ortho}}$, the unconstrained output representations $Z_k^{\text{unc.}} = W_{\text{unc.}} X_k$ and $Z_l^{\text{unc.}} = W_{\text{unc.}} X_l$ can be written as deterministic functions of the orthogonal representations:
\begin{equation}
Z_k^{\text{unc.}} = M \left(W_{\mathrm{ortho}} X_k\right) = M Z_k^{\mathrm{ortho}}, \qquad Z_l^{\text{unc.}} = M \left(W_{\mathrm{ortho}} X_l\right) = M Z_l^{\mathrm{ortho}}
\end{equation}

Since $Z_k^{\text{unc.}}$ depends on $X_k$ solely through $Z_k^{\mathrm{ortho}}$, and $Z_l^{\text{unc.}}$ depends on $X_l$ solely through $Z_l^{\mathrm{ortho}}$, the four random vectors satisfy a symmetric $4$-node Markov chain:
\begin{equation}
Z_k^{\text{unc.}} \longrightarrow Z_k^{\mathrm{ortho}} \longrightarrow Z_l^{\mathrm{ortho}} \longrightarrow Z_l^{\text{unc.}}
\end{equation}

Next, we apply the data processing inequality \citep{Cover2006EleInfoTheory} on the Markov chain to deduce that the mutual information cannot increase.

\paragraph{Step 3: Application of the Data Processing Inequality (DPI).}
By the Data Processing Inequality \citep{Cover2006EleInfoTheory} for Markov chains, passing a random vector through a deterministic mapping cannot increase its mutual information with any other random variable. 

First, applying DPI to the transition $Z_k^{\text{unc.}} \to Z_k^{\mathrm{ortho}} \to Z_l^{\mathrm{ortho}}$:
\begin{equation}
\mathcal{I}\left(Z_k^{\text{unc.}}; Z_l^{\mathrm{ortho}}\right) \;\le\; \mathcal{I}\left(Z_k^{\mathrm{ortho}}; Z_l^{\mathrm{ortho}}\right)
\end{equation}
Second, applying DPI to the right-hand transition $Z_k^{\text{unc.}} \to Z_l^{\mathrm{ortho}} \to Z_l^{\text{unc.}}$:
\begin{equation}
\mathcal{I}\left(Z_k^{\text{unc.}}; Z_l^{\text{unc.}}\right) \;\le\; \mathcal{I}\left(Z_k^{\text{unc.}}; Z_l^{\mathrm{ortho}}\right)
\end{equation}
Combining these two inequalities yields the general weak inequality:
\begin{equation}
\mathcal{I}\left(Z_k^{\text{unc.}}; Z_l^{\text{unc.}}\right) \;\le\; \mathcal{I}\left(Z_k^{\mathrm{ortho}}; Z_l^{\mathrm{ortho}}\right)
\label{eqn:dpi_weak_inequality}
\end{equation}

\paragraph{Step 4: Establishing Strict Inequality via Null-space Information Loss.}
From Proposition 1 and 2, we can state that, under unconstrained training or weight decay, $W_{\text{unc.}}$ undergoes dimensional collapse to an effective rank $r < D_o$. Because $\operatorname{rank}(M) \le \operatorname{rank}(W_{\text{unc.}}) = r$, the operator $M \in \mathbb{R}^{D_o \times D_o}$ is non-invertible and possesses a non-trivial null-space $\operatorname{Null}(M) \subset \mathbb{R}^{D_o}$ of dimension $D_o - r > 0$.

Next, we decompose the orthogonal representation $Z_k^{\mathrm{ortho}}$ into two orthogonal components: $Z_k^{\parallel} \in \operatorname{Range}(M^\top)$ (the preserved subspace) and $Z_k^{\perp} \in \operatorname{Null}(M)$ (the collapsed subspace). We similarly decompose $Z_l^{\mathrm{ortho}}$ as $Z_l^{\mathrm{ortho}} = Z_l^{\parallel} + Z_l^{\perp}$.

The deterministic transformation $M$ acts as a bijection on $\operatorname{Range}(M^\top)$ but maps all variance in $\operatorname{Null}(M)$ to zero (Detailed description in Appendix \ref{sec:appM}):
\begin{equation}
Z_k^{\text{unc.}} = M Z_k^{\parallel} + 0 = M Z_k^{\parallel}, \qquad Z_l^{\text{unc.}} = M Z_l^{\parallel}
\end{equation}
By the chain rule for mutual information:
\begin{equation}
\mathcal{I}\left(Z_k^{\mathrm{ortho}}; Z_l^{\mathrm{ortho}}\right) = \mathcal{I}\left(Z_k^{\parallel}; Z_l^{\parallel}\right) + \mathcal{I}\left(Z_k^{\perp}; Z_l^{\perp} \;\mid\; Z_k^{\parallel}, Z_l^{\parallel}\right)
\end{equation}
Since $M$ is a bijection on $\operatorname{Range}(M^\top)$, $\mathcal{I}\left(Z_k^{\text{unc.}}; Z_l^{\text{unc.}}\right) = \mathcal{I}\left(Z_k^{\parallel}; Z_l^{\parallel}\right)$. Therefore, the exact information lost due to rank collapse is:
\begin{equation}
\mathcal{I}\left(Z_k^{\mathrm{ortho}}; Z_l^{\mathrm{ortho}}\right) - \mathcal{I}\left(Z_k^{\text{unc.}}; Z_l^{\text{unc.}}\right) = \mathcal{I}\left(Z_k^{\perp}; Z_l^{\perp} \;\mid\; Z_k^{\parallel}, Z_l^{\parallel}\right)
\end{equation}

If we assume that, the feature subspace eliminated by the rank collapse of $W_{\text{unc.}}$ carries a non-zero amount of shared statistical information between the two augmented views $X_k$ and $X_l$, making $\mathcal{I}\left(Z_k^{\perp}; Z_l^{\perp} \mid Z_k^{\parallel}, Z_l^{\parallel}\right) > 0$, this establishes the strict inequality:
\begin{equation}
\mathcal{I}\left(Z_k; Z_l \mid WW^\top = I_{D_o}\right) \;>\; \mathcal{I}\left(Z_k; Z_l \mid W_{\text{unc.}}\right)
\end{equation}
which proves that orthogonal regularization strictly preserves mutual information against weight-induced dimensional collapse.

}

\paragraph{Preventing small variance using weight regularization:} Considering the above deductions, we intend to apply weight regularization to the last-layer weight matrix, thereby maintaining full-rank covariance and preserving information across all dimensions, thereby maximising mutual information between embeddings. 

Regularizing the weight matrix via minimizing $\|WW^\top - I\|_F^2$, as follows,
\begin{equation}
    W = \arg \min_{W} \|WW^\top - I\|_F^2
\end{equation}

This objective encourages the singular values of \(W\) to remain bounded away from zero, thereby discouraging degenerate projection directions in the representation space. Consequently, the representation covariance is less likely to collapse onto a lower-dimensional subspace.

\subsection{Optimization Objective}
To prove that our interpretation of the role of the projector in self-supervised contrastive learning is correct, we regularize the weight matrix $W$ of the last layer only by minimizing the regularization loss $\mathcal{L}_{reg} = \lVert WW^T - I\rVert^2$, in addition to the conservative loss in Eqn. \ref{eqn:infonce}. Thus, the final loss is described as,

\begin{minipage}{\linewidth}
\begin{equation}
    \mathcal{L}_{total} = \mathcal{L}_{infonce} + \alpha \mathcal{L}_{reg}
    \label{eqn:finalloss}
\end{equation}
\end{minipage}
}

\textcolor{black}{Based on the above objective function, we optimize the parameters of the neural network following the implementation details outlined in Sec. \ref{sec:impldet}. The results obtained are presented in the section below. We present the results on the datasets CIFAR10, CIFAR100, and ImageNet100 and also present eigenvalue plots for the CIFAR datasets to compare the effect of the proposed regularization towards the prevention of dimensional collapse. For all our experiments, we use $\alpha=0.1$ for optimal performance following WeRank.}

\section{\textcolor{black}{Results and Analyses}}
\label{sec:resanal}
\subsection{Comparison with state-of-the-art Contrastive learning frameworks}

\begin{table}[!ht]
    \centering    
    {\color{black}
    \caption{\textcolor{black}{Comparison of results obtained by applying WeRank variations to SimCLR without Projector, and our proposed strategy on ImageNet100 datasets. Here, `LL' refers to `Last Layer'. (\textcolor{PineGreen}{+}/\textcolor{black}{-}\;$\cdot$): change from the previous model variation. We can observe that the proposed methods outperform the baseline SimCLR (vanilla) and WeRank on almost all cases (except one). This proves the effectiveness of the proposed regularization strategy on preventing degradation of performance due to dimensional collapse}.}
    \label{tab:compinwerank}
    \begin{tabular}{m{6.8cm}|c|c}
    \toprule
    Method & ImageNet100\\\midrule
    SimCLR (vanilla) w/o Projector & \ul{45.82} \\ \midrule 
    SimCLR w/o Proj. + WeRank (Full Enc.) & 44.68 (\textcolor{black}{-1.14}) \\\midrule
    SimCLR w/o Proj. + Wt. Reg. LL (Ours) & \textbf{46.82} (\textcolor{PineGreen}{+1.0}) \\\midrule \midrule
    \end{tabular}
    }
\end{table}

\begin{table}[!ht]
\begin{minipage}{\linewidth}
    \centering    
    \caption{Comparison of results obtained by applying WeRank variations to SimCLR with and without Projector, DirectCLR and our proposed strategy on CIFAR10 and CIFAR100 datasets. Here, `LL' refers to `Last Layer'. `CS' refers to `Constant Subvector'. (\textcolor{PineGreen}{+}/\textcolor{black}{-}\;$\cdot$): change from the \textbf{baseline} model. \textcolor{black}{We can observe that the proposed methods outperform the baseline SimCLR (vanilla) and WeRank on almost all cases (except one). This proves the effectiveness of the proposed regularization strategy on preventing degradation of performance due to dimensional collapse}.}
    \begin{tabular}{m{7.8cm}|c|c}
    \toprule
    Method & CIFAR10 & CIFAR100\\\midrule
    SimCLR (vanilla) & 86.1 & 56.3\\\midrule
    SimCLR + WeRank \citep{pasand2024werank} & \textbf{86.5} (\textcolor{PineGreen}{+0.4}) & \ul{56.8} (\textcolor{PineGreen}{+0.5}) \\\midrule
    SimCLR + Wt. Reg. LL (Ours) & \ul{86.3} (\textcolor{PineGreen}{+0.2}) & \textbf{57.5} (\textcolor{PineGreen}{+1.2}) \\\midrule \midrule
    SimCLR (vanilla) w/o Projector & 84.5 & \ul{52.8} \\ \midrule 
    SimCLR w/o Proj. + WeRank (Full Enc.) & \ul{84.9} (\textcolor{PineGreen}{+0.4}) & 52.6 (\textcolor{black}{-0.2})\\\midrule
    SimCLR w/o Proj. + Wt. Reg. LL (Ours) & \textbf{85.1} (\textcolor{PineGreen}{+0.6}) & \textbf{53.1} (\textcolor{PineGreen}{+0.3})\\\midrule \midrule
    DirectCLR (vanilla) & \textit{85.2} & \textbf{53.2} \\ \midrule
    DirectCLR+ WeRank (Full Enc.) & \ul{85.4} (\textcolor{PineGreen}{+0.2}) & \ul{53.0} (\textcolor{black}{-0.2}) \\\midrule
    DirectCLR+ Wt. Reg. LL (Ours) & \textbf{85.5} (\textcolor{PineGreen}{+0.3}) & \textbf{53.2} (\textcolor{PineGreen}{+0.0})
    \\\midrule \midrule
    SimCLR w/o Proj (w/ CS) & 84.7 & 52.5\\ \midrule
    SimCLR w/o Proj (w/ CS) + WeRank (Full Enc.) & \ul{85.0} (\textcolor{PineGreen}{+0.3}) & \ul{53.0} (\textcolor{PineGreen}{+0.5}) \\\midrule
    SimCLR w/o Proj (w/ CS) + Wt. Reg. LL (Ours) & \textbf{85.5} (\textcolor{PineGreen}{+0.8}) & \textbf{53.1} (\textcolor{PineGreen}{+0.6}) \\\midrule
    \midrule
    \end{tabular}
    \label{tab:compwerank}
\end{minipage}%
\end{table}

In this subsection, we analyse the efficacy of the proposed solution on different self-supervised frameworks, both with and without a projector. From Table \ref{tab:compinwerank} and \ref{tab:compwerank}, we can observe that the proposed solution successfully improves the kNN accuracy of all the SSL frameworks used and also reduces the dimensional collapse issue due to the low variance of feature dimensions in Fig. \ref{fig:comp_no_proj_wt_reg_ll} (in \Cref{sec:plot_no_proj_wt_reg}). \textcolor{black}{We also conduct experiments on the ImageNet100 dataset, and compare our proposed strategy with vanilla SimCLR, and SimCLR added with WeRank, all trained without a projector. We observe from Table \ref{tab:compinwerank} that the proposed strategy outperforms the two baselines on ImageNet100.}

\textcolor{black}{Enforcing $W.W^T=I$ at the last layer improves SimCLR by conditioning the embedding geometry seen by the contrastive objective and preventing projector-induced anisotropy. However, imposing orthogonality constraints throughout the backbone (as in WeRank) over-regularizes the representation function class to a subspace in the function space with orthogonal weight matrices, restricting anisotropic feature shaping that is crucial for semantic abstraction \citep{Wang_2025_nerf, GaoC0Y25anisdf, RudmanE24istar}, and thereby degrades downstream performance. This suggests that spectral conditioning is most effective when applied locally to the contrastive embedding space rather than globally to all feature transformations.}

\subsection{Eigenvalue plots comparing SimCLR Encoder and our method}
\label{sec:plot_no_proj_wt_reg}

\begin{figure*}[h!]
    \hspace*{\fill}%
      \centering
      \begin{subfigure}[b]{0.33\linewidth}
          \centering
        \includegraphics[width = \linewidth]{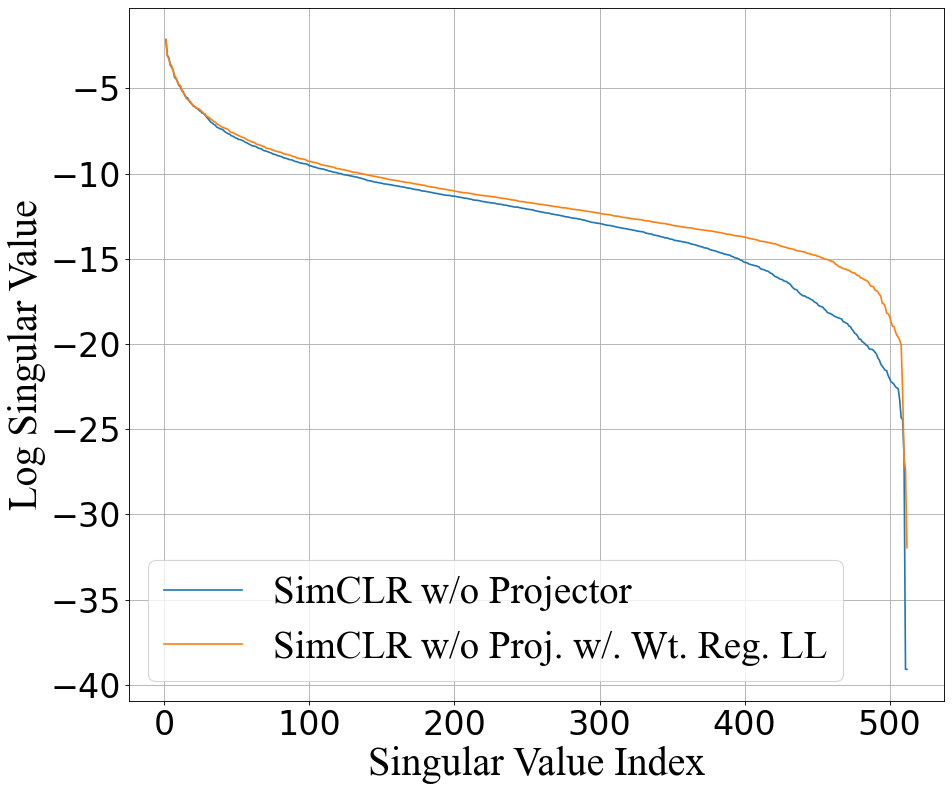}
          \caption{CIFAR10}
          \label{fig:comp_c10}
      \end{subfigure}
      \hfill
      \begin{subfigure}[b]{0.33\linewidth}
          \centering
        \includegraphics[width = \linewidth]{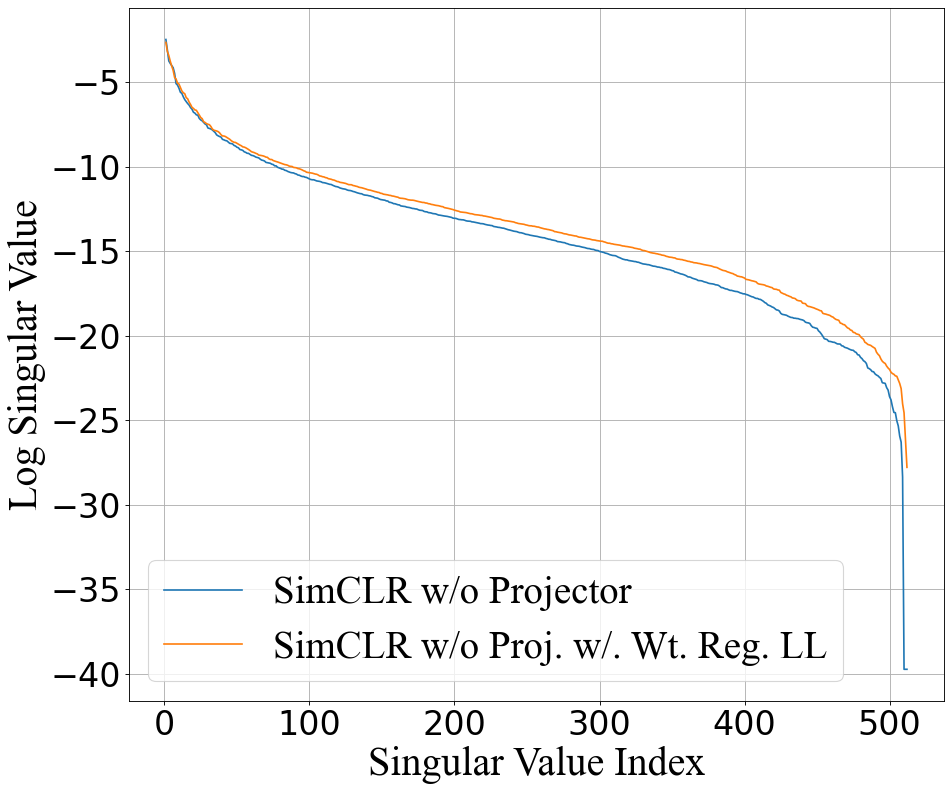}
          \caption{CIFAR100}
          \label{fig:comp_c100}
      \end{subfigure}%
      \hspace*{\fill}%
    \caption{Singular values plots of encoder outputs, embeddings of SimCLR without a projector, and our method (last layer weight regularization) on CIFAR10 and CIFAR100 datasets. \textcolor{black}{The plots clearly exhibit an improvement in the singular values, which indicates that the proposed regularization strategy prevents the dimensional collapse in the encoder output embeddings, when used without an additional projector}.}
      \label{fig:comp_no_proj_wt_reg_ll}
  \end{figure*}

From Fig.\ref{fig:comp_no_proj_wt_reg_ll} we can see that when weight regularization is performed on the last layer of the encoder network, the singular values of the output embeddings improve across all dimensions. Thus, our method can reduce the effect of dimensional collapse due to the low variance of feature dimensions and provide better performance.

\section{Implementation Details}
\label{sec:impldet}
\paragraph{Datasets:}

We primarily used three datasets for our study: CIFAR10, CIFAR100 \citep{krizhevsky2009cifar}, and ImageNet100 \citep{yonglong2020cmc}. 
CIFAR10 and CIFAR100 datasets consist of 10 and 100 classes, respectively, with 50K samples in the training set. ImageNet100 contains 1300 images in each of the 100 classes.

\textbf{Pre-training Details:} For experiments on CIFAR \citep{krizhevsky2009cifar} and ImageNet, we used ResNet18 and ResNet50 as a backbone with the same modifications as done in SimCLR \citep{chen2020simclr} for small-scale datasets (CIFAR). 
We used a batch size of 128 and 256 for CIFAR and ImageNet, respectively.  For the CIFAR and ImageNet datasets, we used SGD and LARS optimizer, respectively. 
All the implementations were done using \textit{lightly-ai} \citep{susmelj2020lightly} library. The value of the temperature hyper-parameter was set to $0.2$ for all experiments. \textcolor{black}{The value of $\alpha$ was set to 0.1 for all experiments, as using a higher value degrades performance. For the CIFAR dataset, the initial learning rate was set to 0.06 and decayed following a cosine schedule, whereas for the ImageNet dataset, an initial learning rate of 0.3 was used and decayed following the same schedule as the CIFAR dataset. For all the datasets, the training was conducted for 200 epochs only.}

For the computation of the SVD decomposition, we simply used the \textit{svd} function from the \textit{numpy} library, following DirectCLR \citep{jing2022directclr}.

{
\section{Limitations and Future Work}

In this work, we presented a layer-wise linear algebraic framework to characterize the multiple possible mechanisms of dimensional collapse, specifically isolating weight norm shrinkage, inter-layer spectral misalignment, and feature covariance misalignment. While our proposed targeted weight regularization ($WW^\top = I$) effectively mitigates weight-induced rank collapse in the last layer, it primarily addresses weight norm decay without explicitly constraining inter-layer or feature-space directional alignments. In future work, we plan to develop a unified optimization framework that jointly addresses spectral and feature-space misalignment, offering a more comprehensive solution to dimensional collapse. Additionally, while our theoretical setup and empirical validations focus specifically on InfoNCE-based contrastive learning architectures with CNN backbones, extending this last-layer rank analysis to non-contrastive self-supervised paradigms (e.g., feature decorrelation \citep{bardes2022vicreg} or predictive architectures \citep{assran2023ijepa, randall2025lejepa}) and Transformer-based backbones (\citep{DosovitskiyB0WZ21VIT} presents an exciting direction for future research.
}

\section{Conclusion}
\label{sec:conclusion}
In this work, we investigate the main reason behind the effectiveness of the projector in preventing dimensional collapse. We analyze mathematically the phenomenon that happens inside the projector and the encoder when trained without a projector. We find that the projector not only creates an information bottleneck but also facilitates the learning of high-level representations in the encoder; without it, a dimensional collapse at the encoder output prevents such learning. We also devise a solution to improve performance by applying weight regularization only to the last layer, with or without a projector, achieving performance better than WeRank, which uses weight regularization across the whole network. We leave the study of the cause of dimensional collapse for our future work.

\subsubsection*{Broader Impact Statement}
This paper presents work whose goal is to advance the field of Machine Learning. There are many potential societal consequences of our work, none of which we feel must be specifically highlighted here.

\bibliographystyle{tmlr}
\bibliography{main}

\newpage
\appendix

{\color{black}

\section{Notations}
\label{subsec:notations}

In this subsection, we discuss the notations followed in the rest of the paper for mathematical derivations and discussions.


\begin{table}[!ht]
    \centering
    {\color{black}
    \caption{Table for Notations}
    \begin{tabular}{|c|c|}
    ine
       Symbols  & What it means \\ ine
        $\mathcal{L}$ & Loss function\\
        $x$ & Input \\
        $f$ & Encoder \\
        $h$ & output of Encoder $f$\\
        $g$ & Projector \\
        $z$ & Output of Projector \\
        $s_{ij}$ & Cosine similarity between the projector output embeddings of the samples of $x_i$ and $x_j$\\
        $B$ & Batch Size \\
        $D$ & Number of dimensions in the encoder output embedding\\
        $d_0$ & Number of trainable dimensions in the encoder output embedding\\
        $d_r$ & Number of non-trainable dimensions in the encoder output embedding\\
        $W_l$ & Weight matrix of $l$-th layer with dimensions $D_o \times D_i$\\
        $D_o$ & Output dimensions \\
        $D_i$ & Input dimensions \\
        $W_l^i$ & $i$-th row of the weight matrix $W_l$\\
        $\mathcal{I}$ & Mutual information \\
    ine
    \end{tabular}
    }
    \label{tab:notationstable}
\end{table}

}

\section{Bounding the cross-covariance term.} 
\label{coundcrosscov}

For any two rows $W^i, W^j \in \mathbb{R}^{1\times D}$ and a symmetric positive semi-definite matrix $\Sigma \in \mathbb{R}^{D\times D}$, the following holds:
\begin{equation}
    |W^j \Sigma (W^i)^T| \le \lambda_{\max}(\Sigma)\,\|W^j\|_2\,\|W^i\|_2.
    \label{eq:cov_bound}
\end{equation}
\textit{Proof.} Let $w_i = (W^i)^T$ and $w_j = (W^j)^T$. 
Then, the scalar quantity can be rewritten as
\begin{equation}
    |W^j \Sigma (W^i)^T| = |w_j^T \Sigma w_i|.
\end{equation}
Since $\Sigma$ is symmetric and positive semi-definite, it admits a spectral factorization $\Sigma = \Sigma^{1/2}\Sigma^{1/2}$.
Hence,
\begin{equation}
    |w_j^T \Sigma w_i|
    = |w_j^T \Sigma^{1/2}\Sigma^{1/2} w_i|
    = |\langle \Sigma^{1/2} w_j,\, \Sigma^{1/2} w_i \rangle|.
\end{equation}
Applying the Cauchy–Schwarz inequality yields
\begin{equation}
    |\langle \Sigma^{1/2} w_j,\, \Sigma^{1/2} w_i \rangle|
    \le \|\Sigma^{1/2} w_j\|_2\,\|\Sigma^{1/2} w_i\|_2.
\end{equation}
Using the sub-multiplicative property of the operator (spectral) norm,
\begin{equation}
    \|\Sigma^{1/2} w_i\|_2 \le \|\Sigma^{1/2}\|_{op}\,\|w_i\|_2,
\end{equation}
and similarly for $w_j$. Therefore,
\begin{equation}
    |w_j^T \Sigma w_i|
    \le \|\Sigma^{1/2}\|_{op}^2\,\|w_j\|_2\,\|w_i\|_2.
\end{equation}
Since $\|\Sigma^{1/2}\|_{op}^2 = \|\Sigma\|_{op} = \lambda_{\max}(\Sigma)$ for symmetric $\Sigma$, we obtain
\begin{equation}
    |W^j \Sigma (W^i)^T|
    \le \lambda_{\max}(\Sigma)\,\|W^j\|_2\,\|W^i\|_2,
\end{equation}
which completes the proof. \hfill $\square$

\noindent
\textbf{Companion lower bounds and alignment condition.}
When \(\Sigma\) is symmetric positive definite (so \(\lambda_{\min}(\Sigma)>0\)), the variance of an output coordinate admits the immediate lower bound
\begin{equation}
    \operatorname{Var}(h_i)=W^i\Sigma (W^i)^T \ge \lambda_{\min}(\Sigma)\,\|W^i\|_2^2.
    \label{eq:var_lower_bound}
\end{equation}
Thus, if \(\operatorname{Var}(h_i)\) is small then \(\|W^i\|\) must be small (cf. Eq. \ref{eqn:boundWi}).

For the cross-covariance term one can write an exact decomposition by using the square root of \(\Sigma\):
\begin{equation}
    W^j\Sigma (W^i)^T = \big\langle \Sigma^{1/2} w_j,\,\Sigma^{1/2} w_i \big\rangle
    = \|\Sigma^{1/2} w_j\|_2\,\|\Sigma^{1/2} w_i\|_2\cos\theta,
    \label{eq:cov_cos}
\end{equation}
where \(w_i=(W^i)^T\), \(w_j=(W^j)^T\), and \(\theta\) is the angle between the vectors \(\Sigma^{1/2} w_j\) and \(\Sigma^{1/2} w_i\) in \(\mathbb{R}^D\). Using \(\|\Sigma^{1/2} w\|_2 \ge \sqrt{\lambda_{\min}(\Sigma)}\,\|w\|_2\) and \(\|\Sigma^{1/2} w\|_2 \le \sqrt{\lambda_{\max}(\Sigma)}\,\|w\|_2\) gives the two-sided inequalities
\begin{equation}
    |W^j\Sigma (W^i)^T|
    = \|\Sigma^{1/2} w_j\|_2\,\|\Sigma^{1/2} w_i\|_2\,|\cos\theta|
    \le \lambda_{\max}(\Sigma)\,\|W^j\|_2\,\|W^i\|_2,
\end{equation}
and, if desired,
\begin{equation}
    |W^j\Sigma (W^i)^T|
    \ge \lambda_{\min}(\Sigma)\,\|W^j\|_2\,\|W^i\|_2\,|\cos\theta|.
    \label{eq:cov_lower_bound}
\end{equation}

Equation \ref{eq:cov_lower_bound} shows why an \emph{alignment} or \emph{non-orthogonality} assumption is necessary to deduce that a small cross-covariance forces small weight norms: even if \(\|\Sigma^{1/2} w_i\|_2\) and \(\|\Sigma^{1/2} w_j\|_2\) are large, the inner product can vanish when \(\cos\theta=0\) (i.e., the two vectors are orthogonal in the \(\Sigma^{1/2}\)-weighted space). Therefore, to conclude that \(|W^j\Sigma (W^i)^T|\) being small implies either \(\|W^i\|\) or \(\|W^j\|\) is small, one must additionally assume a lower bound on \(|\cos\theta|\) (for instance, \(|\cos\theta|\ge c>0\)), or otherwise restrict the class of admissible weight pairs so that they are not \(\Sigma^{1/2}\)-orthogonal.

\hfill \(\square\)


\section{Linear Operator Representation of Convolution}
\label{app:conv_fc}

This appendix shows how a convolutional layer can be expressed as a structured
linear operator by progressively moving from the 1D case to the 2D and
multi-channel settings. This derivation justifies the use of linear-algebraic
arguments in the main text.

\subsection{1D Convolution and Toeplitz Structure}

Consider a 1D convolution without non-linearity.
Let the input be \(x \in \mathbb{R}^{n}\) and the kernel be
\(k = [k_0, \ldots, k_{m-1}]^\top \in \mathbb{R}^{m}\).
Assuming stride 1 and no padding, the output \(y \in \mathbb{R}^{n-m+1}\) is
\begin{equation}
y_i = \sum_{j=0}^{m-1} k_j \, x_{i+j}, \quad i = 0,\ldots,n-m.
\label{eq:conv1d}
\end{equation}

Define the matrix \(W_{\mathrm{1D}} \in \mathbb{R}^{(n-m+1) \times n}\) as
\[
W_{\mathrm{1D}} =
\begin{bmatrix}
k_0 & k_1 & \cdots & k_{m-1} & 0 & \cdots & 0 \\
0 & k_0 & k_1 & \cdots & k_{m-2} & k_{m-1} & \cdots \\
\vdots & & \ddots & & & & \vdots \\
0 & \cdots & 0 & k_0 & k_1 & \cdots & k_{m-1}
\end{bmatrix}.
\]
Then \eqref{eq:conv1d} can be written compactly as
\begin{equation}
y = W_{\mathrm{1D}} \, x.
\label{eq:conv1d_mat}
\end{equation}
The matrix \(W_{\mathrm{1D}}\) is a Toeplitz matrix, i.e., its entries are constant
along each diagonal.

\subsection{2D Convolution and Block Toeplitz Structure}

Now consider a 2D convolution with a single input and output channel.
Let \(X \in \mathbb{R}^{H \times W}\) be the input feature map and
\(K \in \mathbb{R}^{k_h \times k_w}\) be the convolution kernel.
The output \(Y \in \mathbb{R}^{H' \times W'}\) is given by
\begin{equation}
Y_{i,j}
=
\sum_{u=0}^{k_h-1}
\sum_{v=0}^{k_w-1}
K_{u,v} \, X_{i+u,\,j+v}.
\label{eq:conv2d}
\end{equation}

Let \(x = \mathrm{vec}(X)\) and \(y = \mathrm{vec}(Y)\).
Then there exists a matrix \(W_{\mathrm{2D}} \in \mathbb{R}^{(H'W') \times (HW)}\)
such that
\begin{equation}
y = W_{\mathrm{2D}} \, x.
\label{eq:conv2d_mat}
\end{equation}

The matrix \(W_{\mathrm{2D}}\) has a block Toeplitz with Toeplitz blocks (BTTB)
structure:
\[
W_{\mathrm{2D}} =
\begin{bmatrix}
T_0 & T_1 & \cdots & T_{k_h-1} & 0 & \cdots & 0 \\
0 & T_0 & T_1 & \cdots & T_{k_h-2} & T_{k_h-1} & \cdots \\
\vdots & & \ddots & & & & \vdots \\
0 & \cdots & 0 & T_0 & T_1 & \cdots & T_{k_h-1}
\end{bmatrix},
\]
where each block \(T_u \in \mathbb{R}^{W' \times W}\) is a Toeplitz matrix formed
from the \(u\)-th row of the kernel:
\[
T_u =
\begin{bmatrix}
K_{u,0} & K_{u,1} & \cdots & K_{u,k_w-1} & 0 & \cdots \\
0 & K_{u,0} & K_{u,1} & \cdots & K_{u,k_w-1} & \cdots \\
\vdots & & \ddots & & & \vdots
\end{bmatrix}.
\]
Thus, 2D convolution corresponds to a BTTB linear operator.

\subsection{Multi-Channel Convolution}

Finally, consider a multi-channel convolutional layer.
Let
\[
X \in \mathbb{R}^{C_{\mathrm{in}} \times H \times W}, \quad
Y \in \mathbb{R}^{C_{\mathrm{out}} \times H' \times W'},
\]
with kernel
\[
K \in \mathbb{R}^{C_{\mathrm{out}} \times C_{\mathrm{in}} \times k_h \times k_w}.
\]
For output channel \(c\),
\begin{equation}
Y_{c,i,j}
=
\sum_{c'=1}^{C_{\mathrm{in}}}
\sum_{u=0}^{k_h-1}
\sum_{v=0}^{k_w-1}
K_{c,c',u,v} \, X_{c',\,i+u,\,j+v}.
\label{eq:conv_mc}
\end{equation}

Let \(x = \mathrm{vec}(X)\) and \(y = \mathrm{vec}(Y)\).
Then
\begin{equation}
y = W_{\mathrm{conv}} \, x,
\label{eq:conv_mc_mat}
\end{equation}
where
\[
W_{\mathrm{conv}} =
\begin{bmatrix}
W_{1,1} & \cdots & W_{1,C_{\mathrm{in}}} \\
\vdots & \ddots & \vdots \\
W_{C_{\mathrm{out}},1} & \cdots & W_{C_{\mathrm{out}},C_{\mathrm{in}}}
\end{bmatrix}.
\]
Each block \(W_{c,c'}\) is a BTTB matrix constructed from the kernel slice
\(K_{c,c'}\) as in the single-channel 2D case.

Each row of \(W_{\mathrm{conv}}\) corresponds to a specific output channel and
spatial location and contains exactly
\(C_{\mathrm{in}} k_h k_w\) non-zero entries given by the kernel weights.
Consequently,
\begin{equation}
\|w_i\|_2^2
=
\sum_{c'=1}^{C_{\mathrm{in}}}
\sum_{u=0}^{k_h-1}
\sum_{v=0}^{k_w-1}
K_{c,c',u,v}^2.
\label{eq:row_norm}
\end{equation}

\subsection{Covariance and Variance Propagation}

Let \(\Sigma_x = \mathrm{Cov}(x)\). From~\eqref{eq:conv_mc_mat},
\begin{equation}
\Sigma_y = \mathrm{Cov}(y)
= W_{\mathrm{conv}}\, \Sigma_x\, W_{\mathrm{conv}}^\top.
\label{eq:cov_prop}
\end{equation}
For the \(i\)-th output unit,
\begin{equation}
\mathrm{Var}(y_i) = w_i^\top \Sigma_x w_i,
\label{eq:var_prop}
\end{equation}
where \(w_i^\top\) denotes the \(i\)-th row of \(W_{\mathrm{conv}}\).
These expressions justify the application of linear-algebraic arguments
(e.g., rank, nullspace, and norm-based reasoning) to convolutional layers.

\subsection{Remark on Tensor Representation}

The vectorization of \(X\) and \(Y\) is introduced solely for analytical
purposes. The tensors \(X\) and \(Y\) themselves are not modified, and all
empirical results and interpretations in the main text are expressed in terms
of the original multi-dimensional feature maps.

\section{Decomposition of Entropy and Mutual Information}
\label{app_mi}

\paragraph{Entropy of a random vector.}
Let \( \mathbf Z = (z_1, z_2, \dots, z_D) \) be a continuous random vector with joint density
\( p(z_1,\dots,z_D) \).
By definition, the (differential) entropy of \( \mathbf Z \) is
\[
H(\mathbf Z) = -\int p(z_1,\dots,z_D)\, \log p(z_1,\dots,z_D)\, dz_1\cdots dz_D .
\]

Using the chain rule of probabilities,
\[
p(z_1,\dots,z_D) = p(z_1)\prod_{i=2}^D p(z_i \mid z_{1:i-1}),
\]
and therefore
\[
\log p(z_1,\dots,z_D) = \log p(z_1) + \sum_{i=2}^D \log p(z_i \mid z_{1:i-1}).
\]

Substituting into the entropy definition,
\[
\begin{aligned}
H(\mathbf Z) &= -\int p(z_1,\dots,z_D) \left[ \log p(z_1) + \sum_{i=2}^D \log p(z_i \mid z_{1:i-1}) \right] dz_1\cdots dz_D\\ 
&= -\int p(z_1)\log p(z_1)\,dz_1 - \sum_{i=2}^D \int p(z_1,\dots,z_i)\log p(z_i \mid z_{1:i-1})\,dz_1\cdots dz_i \\
&= H(z_1) + \sum_{i=2}^D H(z_i \mid z_{1:i-1}).
\end{aligned}
\]

Thus, the entropy of the vector admits the chain-rule decomposition
\[
\boxed{
H(\mathbf Z) = \sum_{i=1}^D H(z_i \mid z_{1:i-1}),
}
\]
with the convention \(H(z_1 \mid z_{1:0}) \equiv H(z_1)\).





\paragraph{Setup for Mutual Information Derivation in terms of decomposed Entropy}
Let
\[
\mathbf Z_1 = (z_{11}, z_{12}, \dots, z_{1D}), 
\qquad
\mathbf Z_2 = (z_{21}, z_{22}, \dots, z_{2D})
\]
be two continuous random vectors with a joint density
\(p(\mathbf Z_1,\mathbf Z_2)\).
No independence assumptions are made across dimensions.

\paragraph{Joint Entropy of the two vectors $\mathbf Z_1$ and $\mathbf Z_2$.}
By definition,
\[
H(\mathbf Z_1,\mathbf Z_2)
=
-\int p(\mathbf Z_1,\mathbf Z_2)
\log p(\mathbf Z_1,\mathbf Z_2)
\,d\mathbf Z_1\,d\mathbf Z_2 .
\]

Using the chain rule of probabilities,
\[
p(\mathbf Z_1,\mathbf Z_2)
=
p(\mathbf Z_1)\,p(\mathbf Z_2 \mid \mathbf Z_1),
\]
hence
\[
\log p(\mathbf Z_1,\mathbf Z_2)
=
\log p(\mathbf Z_1)
+
\log p(\mathbf Z_2 \mid \mathbf Z_1).
\]

Substituting into the entropy definition,
\[
\begin{aligned}
H(\mathbf Z_1,\mathbf Z_2)
&=
-\int p(\mathbf Z_1,\mathbf Z_2)
\big[
\log p(\mathbf Z_1)
+
\log p(\mathbf Z_2 \mid \mathbf Z_1)
\big]
\,d\mathbf Z_1\,d\mathbf Z_2
\\
&=
-\int p(\mathbf Z_1)\log p(\mathbf Z_1)\,d\mathbf Z_1
-
\int p(\mathbf Z_1,\mathbf Z_2)
\log p(\mathbf Z_2 \mid \mathbf Z_1)
\,d\mathbf Z_1\,d\mathbf Z_2
\\
&=
H(\mathbf Z_1)
+
H(\mathbf Z_2 \mid \mathbf Z_1).
\end{aligned}
\]

\paragraph{Entropy of each vector.}
From the vector entropy derivation,
\[
H(\mathbf Z_1)
=
\sum_{i=1}^D H(z_{1i} \mid z_{1,1:i-1}),
\qquad
H(\mathbf Z_2)
=
\sum_{j=1}^D H(z_{2j} \mid z_{2,1:j-1}).
\]

\paragraph{Definition of mutual information.}
The mutual information between the two vectors is
\[
\mathcal{I}(\mathbf Z_1;\mathbf Z_2)
=
H(\mathbf Z_1) + H(\mathbf Z_2) - H(\mathbf Z_1,\mathbf Z_2).
\]

Substituting the joint entropy decomposition,
\[
\begin{aligned}
\mathcal{I}(\mathbf Z_1;\mathbf Z_2)
&=
H(\mathbf Z_1) + H(\mathbf Z_2)
-
\big[
H(\mathbf Z_1) + H(\mathbf Z_2 \mid \mathbf Z_1)
\big]
=
H(\mathbf Z_2) - H(\mathbf Z_2 \mid \mathbf Z_1).
\end{aligned}
\]

\paragraph{Step 5: Decomposition over dimensions of \(\mathbf Z_2\).}
Using the entropy chain rule for \(\mathbf Z_2\),
\[
H(\mathbf Z_2 \mid \mathbf Z_1)
=
\sum_{j=1}^D
H(z_{2j} \mid \mathbf Z_1, z_{2,1:j-1}),
\]
hence
\[
\begin{aligned}
\mathcal{I}(\mathbf Z_1;\mathbf Z_2)
&=
\sum_{j=1}^D
\Big[
H(z_{2j} \mid z_{2,1:j-1})
-
H(z_{2j} \mid \mathbf Z_1, z_{2,1:j-1})
\Big]
=
\boxed{
\sum_{j=1}^D
\mathcal{I}\left(
z_{2j}; \mathbf Z_1 \mid z_{2,1:j-1}
\right)
}
\end{aligned}
\]

\paragraph{Symmetric form.}
By symmetry,
\[
\boxed{
\mathcal{I}(\mathbf Z_1;\mathbf Z_2)
=
\sum_{i=1}^D
\mathcal{I}\left(
z_{1i}; \mathbf Z_2 \mid z_{1,1:i-1}
\right).
}
\]

\paragraph{Step 6: Chain rule for conditional mutual information.}

Applying the chain rule to \(\mathbf Z_1=(z_{11},\dots,z_{1D})\),
\[
\mathcal{I}(z_{2j}; \mathbf Z_1 \mid z_{2,1:j-1})
=
\sum_{i=1}^D
\mathcal{I}\left(
z_{2j}; z_{1i}
\;\middle|\;
z_{2,1:j-1},\, z_{1,1:i-1}
\right).
\]

\paragraph{Final decomposition.}
\[
\boxed{
\mathcal{I}(\mathbf Z_1;\mathbf Z_2)
=
\sum_{j=1}^D
\sum_{i=1}^D
\mathcal{I}\left(
z_{2j}; z_{1i}
\;\middle|\;
z_{2,1:j-1},\, z_{1,1:i-1}
\right)
}
\]

\section{Why $M$ Maps All Variance in $\operatorname{Null}(M)$ to Zero and Acts as a Bijection on $\operatorname{Range}(M^\top)$}
\label{sec:appM}

This comes directly from the Fundamental Theorem of Linear Algebra and the orthogonal decomposition of vector spaces. 

\paragraph{1. Why $M$ Maps All Variance in $\operatorname{Null}(M)$ to Zero}

By definition, the nullspace (or kernel) of a matrix $M \in \mathbb{R}^{D_o \times D_o}$ consists of all vectors that $M$ evaluates to zero:
$$\operatorname{Null}(M) = \{\mathbf{v} \in \mathbb{R}^{D_o} \mid M \mathbf{v} = \mathbf{0}\}$$

If a random vector $Z^\perp$ lies entirely in $\operatorname{Null}(M)$, then every realization of $Z^\perp$ satisfies $M Z^\perp = \mathbf{0}$.Since $M Z^\perp$ is identically the constant zero vector $\mathbf{0}$:Its mean is $\mathbb{E}[M Z^\perp] = \mathbf{0}$.

Its covariance matrix is $\operatorname{Cov}(M Z^\perp) = M \operatorname{Cov}(Z^\perp) M^\top = \mathbf{0}$. Thus, $M$ completely extinguishes all variance along these directions—collapsing those dimensions into a single point at the origin.

\paragraph{2. Why $M$ Acts as a Bijection on $\operatorname{Range}(M^\top)$}

By the Fundamental Theorem of Linear Algebra, the space $\mathbb{R}^{D_o}$ decomposes into two mutually orthogonal, complementary subspaces:
$$\mathbb{R}^{D_o} = \operatorname{Range}(M^\top) \oplus \operatorname{Null}(M)$$

Now consider the linear mapping $M$ restricted exclusively to vectors in $\operatorname{Range}(M^\top)$, i.e., $M\vert_{\operatorname{Range}(M^\top)}: \operatorname{Range}(M^\top) \to \operatorname{Range}(M)$. 

To prove it is a bijection (both injective and surjective):

\begin{itemize}
    \item \textbf{A. Injectivity (One-to-One)}
    
    Suppose two vectors $\mathbf{v}_1, \mathbf{v}_2 \in \operatorname{Range}(M^\top)$ produce the same output under $M$:
$$M \mathbf{v}_1 = M \mathbf{v}_2 \implies M(\mathbf{v}_1 - \mathbf{v}_2) = \mathbf{0}$$

This implies that the difference vector $(\mathbf{v}_1 - \mathbf{v}_2)$ lies in $\operatorname{Null}(M)$. However, since $\mathbf{v}_1, \mathbf{v}_2 \in \operatorname{Range}(M^\top)$, their difference $(\mathbf{v}_1 - \mathbf{v}_2)$ must also belong to $\operatorname{Range}(M^\top)$.Because $\operatorname{Range}(M^\top)$ and $\operatorname{Null}(M)$ are orthogonal complements, their intersection contains only the zero vector:
$$\operatorname{Range}(M^\top) \cap \operatorname{Null}(M) = \{\mathbf{0}\} \implies \mathbf{v}_1 - \mathbf{v}_2 = \mathbf{0} \implies \mathbf{v}_1 = \mathbf{v}_2$$

This proves that $M$ is strictly injective on $\operatorname{Range}(M^\top)$ (no two distinct inputs in $\operatorname{Range}(M^\top)$ map to the same output).

\item \textbf{B. Surjectivity (Onto)}

By the Rank-Nullity Theorem, the dimension of the row space equals the dimension of the column space (the rank $r$):
$$\dim(\operatorname{Range}(M^\top)) = \operatorname{rank}(M) = \dim(\operatorname{Range}(M)) = r$$

An injective linear map between two finite-dimensional vector spaces of identical dimension $r$ is automatically surjective (onto $\operatorname{Range}(M)$).
\end{itemize}

This is important for evaluating Mutual Information because $M$ is a bijection on $\operatorname{Range}(M^\top)$, passing $Z^\parallel \in \operatorname{Range}(M^\top)$ through $M$ is a lossless, invertible transformation. Invertible linear mappings preserve mutual information perfectly:$$\mathcal{I}(M Z_k^\parallel; M Z_l^\parallel) = \mathcal{I}(Z_k^\parallel; Z_l^\parallel)$$Meanwhile, because $M$ maps $\operatorname{Null}(M)$ to $\mathbf{0}$, all information contained in $Z^\perp \in \operatorname{Null}(M)$ is permanently wiped out. This isolates the exact source of information loss to $\mathcal{I}(Z_k^\perp; Z_l^\perp \mid Z_k^\parallel, Z_l^\parallel)$.

\end{document}